\documentclass{article}

\usepackage{algorithm}
\usepackage[algo2e, vlined, ruled, noend]{algorithm2e}
\usepackage{algorithmic}
\usepackage{amsmath}
\usepackage{amssymb}
\usepackage{amsthm}
\usepackage[english]{babel}
\usepackage[singlelinecheck=false]{caption}
\usepackage[colorlinks=true,citecolor=blue,linkcolor=blue]{hyperref}
\usepackage{natbib}
\usepackage{thmtools}
\usepackage{tikz}
\usepackage{color}
\usepackage{comment}
\usepackage{commath}
\usepackage{dsfont}
\usepackage{framed}
\usepackage{caption, subcaption}
\usepackage{authblk}
\usepackage{times}
\usepackage[accepted]{icml2019}  % camera ready version
\usepackage[outdir=./]{epstopdf}
\usetikzlibrary{arrows}
\usetikzlibrary{angles}
\usetikzlibrary{backgrounds}
\usetikzlibrary{calc}
\usetikzlibrary{cd}  % commutative diagrams
\usetikzlibrary{decorations.markings}
\usetikzlibrary{intersections}
\usetikzlibrary{patterns}
\usetikzlibrary{shapes}
\usetikzlibrary{through}

\newtheorem{definition}{Definition}
\newtheorem{proposition}[definition]{Proposition}
\newtheorem{lemma}[definition]{Lemma}
\newtheorem{theorem}[definition]{Theorem}
\newtheorem{corollary}[definition]{Corollary}

\newcommand{\calB}{\mathcal{B}}
\newcommand{\calA}{\mathcal{A}}
\newcommand{\calF}{\mathcal{F}}
\newcommand{\calH}{\mathcal{H}}
\newcommand{\one}{\mathds{1}}
\newcommand{\R}{\mathbb{R}}  % set of real numbers
\newcommand{\indicator}[1]{\one\left[#1 \right]} % indicator
\newcommand{\ip}[2]{\left\langle #1, #2 \right\rangle} % inner product
\newcommand{\e}{\mathbf{e}}
\newcommand{\zero}{\mathbf{0}}

\DeclareMathOperator{\B}{B}
\DeclareMathOperator*{\argmax}{argmax}
\DeclareMathOperator*{\argmin}{argmin}
\DeclareMathOperator*{\Exp}{\mathbf{E}}  % expected value
\DeclareMathOperator*{\polylog}{polylog}
\DeclareMathOperator*{\Vol}{Vol}

\definecolor{darkgreen}{rgb}{0,0.5,0}
\definecolor{darkred}{rgb}{0.7,0,0}
\definecolor{teal}{rgb}{0.3,0.8,0.8}
\definecolor{orange}{rgb}{1.0,0.5,0.0}
\definecolor{purple}{rgb}{0.8,0.0,0.8}
\definecolor{blue}{rgb}{0.0,0.0,1.0}

\newcounter{protocol}
\makeatletter
\newenvironment{protocol}[1][htb]{%
  \let\c@algorithm\c@protocol
  \renewcommand{\ALG@name}{Protocol}% Update algorithm name
  \begin{algorithm}[#1]%
  }{\end{algorithm}
}
\makeatother

\begin{document}

% The \icmltitle you define below is probably too long as a header.
% Therefore, a short form for the running title is supplied here:
%\icmltitlerunning{Bandit Multiclass Linear Classification: Efficient Algorithms for the Separable Case}

%\twocolumn[
%\icmltitle{Bandit Multiclass Linear Classification: \\ Efficient Algorithms for the Separable Case}

\icmltitlerunning{Bandit Multiclass Linear Classification: Efficient Algorithms for the Separable Case}

\twocolumn[
\icmltitle{Bandit Multiclass Linear Classification: \\ Efficient Algorithms for the Separable Case}

\begin{icmlauthorlist}
\icmlauthor{Alina Beygelzimer\textsuperscript{*}}{yahoo}
\icmlauthor{D\'avid P\'al\textsuperscript{*}}{yahoo}
\icmlauthor{Bal\'azs Sz\"or\'enyi\textsuperscript{*}}{yahoo}
\icmlauthor{Devanathan Thiruvenkatachari\textsuperscript{*}}{nyu}
\icmlauthor{Chen-Yu Wei\textsuperscript{*}}{usc}
\icmlauthor{Chicheng Zhang\textsuperscript{*}}{microsoft}

\end{icmlauthorlist}

\icmlaffiliation{yahoo}{Yahoo Research, New York, NY, USA}
\icmlaffiliation{nyu}{New York University, New York, NY, USA}
\icmlaffiliation{usc}{University of Southern California, Los Angeles, CA, USA}
\icmlaffiliation{microsoft}{Microsoft Research, New York, NY, USA}

\icmlcorrespondingauthor{D\'avid P\'al}{davidko.pal@gmail.com}

\icmlkeywords{multi-armed bandits, contextual bandits, online classification, linear separability}

\vskip 0.3in
]

\printAffiliationsAndNotice{\textsuperscript{*}The authors are listed in alphabetical order.}

\begin{abstract}
We study the problem of efficient online multiclass linear classification with
bandit feedback, where all examples belong to one of $K$ classes and lie in the
$d$-dimensional Euclidean space. Previous works have left open the challenge of
designing efficient algorithms with finite mistake bounds when the data is
linearly separable by a margin $\gamma$. In this work, we take a first step
towards this problem. We consider two notions of linear separability,
\emph{strong} and \emph{weak}.

\begin{enumerate}
\item Under the strong linear separability condition, we design an efficient
algorithm that achieves a near-optimal mistake bound of
$O\left( K/\gamma^2 \right)$.

\item Under the more challenging weak linear separability condition, we design
an efficient algorithm with a mistake bound of $\min (2^{\widetilde{O}(K \log^2
(1/\gamma))}, 2^{\widetilde{O}(\sqrt{1/\gamma} \log K)})$.\footnote{We use the
notation $\widetilde{O}(f(\cdot)) = O(f(\cdot) \polylog(f(\cdot)))$.} Our
algorithm is based on kernel Perceptron and is inspired by the work of
\citet{Klivans-Servedio-2008} on improperly learning intersection of halfspaces.
\end{enumerate}
\end{abstract}

\section{Introduction}
\label{section:introduction}

We study the problem of \textsc{Online Multiclass Linear Classification with
Bandit Feedback}~\citep{Kakade-Shalev-Shwartz-Tewari-2008}. The problem can be
viewed as a repeated game between a learner and an adversary. At each time step
$t$, the adversary chooses a labeled example $(x_t, y_t)$ and reveals the
feature vector $x_t$ to the learner. Upon receiving $x_t$, the learner makes a
prediction $\widehat{y}_t$ and receives feedback. In contrast with the standard
full-information setting, where the feedback given is the correct label $y_t$,
here the feedback is only a binary indicator of whether the prediction was
correct or not. The protocol of the problem is formally stated below.

\begin{protocol}[h]
\caption{\textsc{Online Multiclass Linear Classification with Bandit Feedback}
\label{algorithm:game-protocol}}
\textbf{Require:} Number of classes $K$, number of rounds $T$. \\
\textbf{Require:} Inner product space $(V,\ip{\cdot}{\cdot})$. \\
\For{$t=1,2,\dots,T$}{
Adversary chooses example $(x_t, y_t) \in V \times \{1,2,\dots,K\}$ where $x_t$ is revealed to the learner.\\
Predict class label $\widehat y_t \in \{1,2,\dots,K\}$.\\
Observe feedback $z_t = \indicator{\widehat y_t \neq y_t} \in \{0,1\}$.
}
\end{protocol}

The performance of the learner is measured by its cumulative number of
mistakes $\sum_{t=1}^T z_t = \sum_{t=1}^T \indicator{\widehat y_t \neq y_t}$,
where $\one$ denotes the indicator function.

In this paper, we focus on the special case when the examples chosen by the
adversary lie in $\R^d$ and are linearly separable with a margin. We introduce
two notions of linear separability, \emph{weak} and \emph{strong}, formally
stated in \autoref{definition:linear-separability}. The standard notion of
multiclass linear separability~\citep{Crammer-Singer-2003} corresponds to the
weak linear separability. For multiclass classification with $K$ classes, weak
linear separability requires that all examples from the same class lie in an
intersection of $K-1$ halfspaces and all other examples lie in the complement of
the intersection of the halfspaces. Strong linear separability means that
examples from each class are separated from the remaining examples by a
\emph{single} hyperplane.

In the full-information feedback setting, it is well known
\citep{Crammer-Singer-2003} that if all examples have norm at most $R$ and are
weakly linearly separable with a margin $\gamma$, then the \textsc{Multiclass
Perceptron} algorithm makes at most $\lfloor 2(R/\gamma)^2 \rfloor$ mistakes. It
is also known that any (possibly randomized) algorithm must make $\frac{1}{2}
\left\lfloor (R/\gamma)^2 \right \rfloor$ mistakes in the worst case. The
\textsc{Multiclass Perceptron} achieves an information-theoretically optimal
mistake bound, while being time and memory efficient.\footnote{We call an algorithm
 computationally efficient, if its running time is polynomial in $K$, $d$, $1/\gamma$
 and $T$.}
\footnote{For
completeness, we present these folklore results along with their proofs in
Appendix~\ref{section:multiclass-perceptron-proofs} in the supplementary
material.}

The bandit feedback setting, however, is much more challenging. For the
strongly linearly separable case, we are not aware of any
prior efficient algorithm with a finite mistake bound.
~\footnote{Although~\citet{Chen-Chen-Zhang-Chen-Zhang-2009}
claimed that their Conservative OVA algorithm with PA-I update has a finite
mistake bound under the strong linear separability condition, their
Theorem~2 is incorrect: first, their Lemma~1 (with $C = +\infty$) along with their Theorem~1
implies a mistake upper bound of
$(\frac{R}{\gamma})^2$, which contradicts the lower
bound in our Theorem~\ref{theorem:strongly-separable-examples-mistake-lower-bound};
second, their Lemma~1 cannot be directly applied to the bandit feedback setting.}
We design a simple and
efficient algorithm
(Algorithm~\ref{algorithm:algorithm-for-strongly-linearly-separable-examples})
that makes at most $O(K (R/\gamma)^2)$ mistakes in expectation. Its memory
complexity and per-round time complexity are both $O(dK)$. The algorithm can be
viewed as running $K$ copies of the \textsc{Binary Perceptron} algorithm, one
copy for each class. We prove that any (possibly randomized) algorithm must make
$\Omega(K (R/\gamma)^2)$ mistakes in the worst case. The extra $O(K)$
multiplicative factor in the mistake bound, as compared to the full-information
setting, is the price we pay for the bandit feedback, or more precisely, the
lack of full-information feedback.

For the case when the examples are weakly linearly separable, it was open for a
long time whether there exist \textit{efficient} algorithms with finite mistake
bound~\citep{Kakade-Shalev-Shwartz-Tewari-2008, Beygelzimer-Orabona-Zhang-2017}.
Furthermore, \citet{Kakade-Shalev-Shwartz-Tewari-2008} ask the question:
Is there \textit{any} algorithm with a finite mistake bound that has no explicit
dependence on the dimensionality of the feature vectors? We answer both
questions affirmatively by providing an efficient algorithm with finite
dimensionless mistake bound (Algorithm~\ref{algorithm:kernelized}).\footnote{An
inefficient algorithm was given by~\cite{Daniely-Helbertal-2013}.}

The strategy used in Algorithm~\ref{algorithm:kernelized} is to construct a
non-linear feature mapping $\phi$ and associated positive definite kernel
$k(x,x')$ that makes the examples \emph{strongly} linearly separable in a
higher-dimensional space. We then use the kernelized version of
Algorithm~\ref{algorithm:algorithm-for-strongly-linearly-separable-examples} for
the strongly separable case. The kernel $k(x,x')$ corresponding to the feature
mapping $\phi$ has a simple explicit formula and can be computed in $O(d)$ time,
making Algorithm~\ref{algorithm:kernelized} computationally efficient. For
details on kernel methods see e.g.~\citep{Scholkopf-Smola-2002} or
\citep{Shawe-Taylor-Cristianini-2004}.

The number of mistakes of the kernelized algorithm depends on the margin in the
corresponding feature space. We analyze how the mapping $\phi$ transforms the
margin parameter of weak separability in the original space $\R^d$ into a margin
parameter of strong separability in the new feature space. This problem is
related to the problem of learning intersection of halfspaces and has been
studied previously by \citet{Klivans-Servedio-2008}. As a side result, we
improve on the results of \citet{Klivans-Servedio-2008} by removing the
dependency on the original dimension $d$.

The resulting kernelized algorithm runs in time polynomial in the
original dimension of the feature vectors $d$, the number of classes $K$, and
the number of rounds $T$. We prove that if the examples lie in the unit ball of
$\R^d$ and are weakly linearly separable with margin $\gamma$,
Algorithm~\ref{algorithm:kernelized} makes at
most $\min (2^{\widetilde{O}(K \log^2
(1/\gamma))}, 2^{\widetilde{O}(\sqrt{1/\gamma} \log K)})$ mistakes.

In Appendix~\ref{section:nearest-neighbor-algorithm}, we propose and analyze a
very different algorithm for weakly linearly separable data. The algorithm is
based on the obvious idea that two points that are close enough must have the
same label.
%\chicheng{I wonder if the mistake bound of this algorithm is too large to be included
%as ``our contributions'' - perhaps we should only mention it as a folklore/baseline?}

Finally, we study two questions related to the computational and
information-theoretic hardness of the problem. Any algorithm for the bandit
setting collects information in the form of so called \emph{strongly labeled}
and \emph{weakly labeled} examples. Strongly labeled examples are those for
which we know the class label. Weakly labeled example is an example for which we
know that class label can be anything except for one particular class. In
Appendix~\ref{section:np-hardness-of-weak-labeling-problem}, we show that the
offline problem of finding a multiclass linear classifier consistent with a set
of strongly and weakly labeled examples is NP-hard. In
Appendix~\ref{section:mistake-lower-bound-for-ignorant-algorithms}, we prove a
lower bound on the number of mistakes of any algorithm that uses only
strongly-labeled examples and ignores weakly labeled examples.

\section{Related work}
\label{section:related-work}

The problem of online bandit multiclass learning was initially formulated in the
pioneering work of~\citet{Auer-Long-1999} under the name of ``weak reinforcement
model''. They showed that if all examples agree with some classifier from a
prespecified hypothesis class $\calH$, then the optimal mistake bound in the
bandit setting can be upper bounded by the optimal mistake bound in the full
information setting, times a factor of $(2.01 + o(1))K \ln K$. \citet{Long-2017}
later improved the factor to $(1 + o(1)) K \ln K$ and showed its
near-optimality. \citet{Daniely-Helbertal-2013} extended the results to the
setting where the performance of the algorithm is measured by its regret, i.e.
the difference between the number of mistakes made by the algorithm and the
number of mistakes made by the best classifier in $\calH$ in hindsight. We
remark that all algorithms developed in this context are computationally
inefficient.

The linear classification version of this problem is initially studied
by~\citet{Kakade-Shalev-Shwartz-Tewari-2008}. They proposed two computationally
inefficient algorithms that work in the weakly linearly separable setting, one
with a mistake bound of $O(K^2 d \ln(d/\gamma))$, the other with a mistake bound
of $\widetilde{O}((K^2/\gamma^2) \ln T)$. The latter result was later improved
by \citet{Daniely-Helbertal-2013}, which gives a computationally inefficient
algorithm with a mistake upper bound of $\widetilde{O}(K/\gamma^2)$. In
addition,~\citet{Kakade-Shalev-Shwartz-Tewari-2008} propose the
\textsc{Banditron} algorithm, a computationally efficient algorithm that has a
$O(T^{2/3})$ regret against the multiclass hinge loss in the general setting,
and has a $O(\sqrt{T})$ mistake bound in the $\gamma$-weakly linearly separable
setting. In contrast to mild dependencies on the time horizon for mistake bounds
of computationally inefficient algorithms, the polynomial dependence of
\textsc{Banditron}'s mistake bound on the time horizon is undesirable for
problems with a long time horizon, in the weakly linearly separable setting. One
key open question left by~\citet{Kakade-Shalev-Shwartz-Tewari-2008} is whether
one can design computationally efficient algorithms that achieve mistake bounds
that match or improve over those of inefficient algorithms. In this paper, we
take a step towards answering this question, showing that efficient algorithms
with mistake bounds quasipolynomial in $1/\gamma$ (for constant $K$) and
quasipolynomial in $K$ (for constant $\gamma$) can be obtained.

%In addition, it also
% shows that under different relationships between $k$, $d$ and $\gamma$,
% the adversary can force the learner to incur $\Omega(K^2 d)$
% or $\Omega(\frac{K}{\gamma^2})$ mistakes.

%Whether one can design an efficient algorithm with a finite mistake bound that
%has no dimensionality dependence is mentioned as an open problem in
%~\cite{Kakade-Shalev-Shwartz-Tewari-2008}, where in this paper we provide a
%positive answer. $O(k^2 d \ln(\frac{d}{\gamma^2}))$

The general problem of linear bandit multiclass learning has received considerable attention~\citep{Abernethy-Rakhlin-2009, Wang-Jin-Valizadegan-2010,
Crammer-Gentile-2013, Hazan-Kale-2011, Beygelzimer-Orabona-Zhang-2017,
Foster-Kale-Luo-Mohri-Sridharan-2018}. \citet{Chen-Lin-Lu-2014,
Zhang-Jung-Tewari-2018} study online bandit multiclass boosting under bandit
feedback, where one can view boosting as linear classification by treating each
base hypothesis as a separate feature.
%However, most of these works achieve a
%regret of order $\widetilde{O}(\sqrt{T})$ to $O(T^{3/4})$.
In the weakly linearly separable setting, however,
these algorithms can only guarantee a mistake
bound of ${O}(\sqrt{T})$ at best.

The problem considered here is a special case of the
contextual bandit problem~\citep{Auer-2003, Langford-Zhang-2008}.
In this general problem, there is a hidden cost vector $c_t$ associated with every prediction in round $t$.  Upon receiving $x_t$ and predicting $\widehat{y}_t \in \{1,\ldots,K\}$, the learner gets to observe the incurred cost $c_t(\widehat{y}_t)$.
The goal of the learner is to minimize its regret with respect to the best predictor in some predefined policy class $\Pi$, given by $\sum_{t=1}^T c_t(\widehat{y}_t) -
\min_{\pi \in \Pi} \sum_{t=1}^T c_t(\pi(x_t))$.
Bandit multiclass learning is a special case
where the cost $c_t(i)$ is the classification error $\indicator{i \neq y_t}$ and
the policy class is the set of linear classifiers $\cbr{x \mapsto \argmax_y (Wx)_y: W \in \R^{K \times d}}$.
There has been significant progress on the general contextual bandit problem
assuming access to an optimization oracle that returns a policy in $\Pi$ with the smallest total cost on any given set of cost-sensitive examples~\citep{Dudik-Hsu-Kale-Karampatziakis-Langford-Reyzin-Zhang-2011,
Agarwal-Hsu-Kale-Langford-Li-Schapire-2014, Rakhlin-Sridharan-2016,
Syrgkanis-Krishnamurthy-Schapire-2016,
Syrgkanis-Luo-Krishnamurthy-Schapire-2016}.
However, such an oracle abstracting efficient search through $\Pi$ is generally not available
in our setting due to computational hardness results~\citep{Arora-Babai-Stern-Sweedyk-1997}.

%While solving a more general problem without making assumptions on the structure of the cost vector $c_t$ or the policy class, these results assume that $x_t$ vectors or $(x_t,c_t)$ pairs are generated i.i.d. and that such an oracle abstracting efficient search through $\Pi$ is available---neither of which we assume here.

%While solving a more general problem without making assumptions on the structure of the cost vector $c_t$ or the policy class,
%However, these algorithms are not
%truly computationally efficient in our setting, as it is NP-hard in general to
%find a linear classifier that minimizes the empirical cost over a set of
%examples~\citep{Arora-Babai-Stern-Sweedyk-1997}.

Recently, \citet{Foster-Krishnamurthy-2018} developed a rich theory of
contextual bandits with surrogate losses, focusing on regrets of the form
$\sum_{t=1}^T c_t(\widehat{y}_t) - \min_{f \in \calF} \sum_{t=1}^T \frac{1}{K}
\sum_{i=1}^K c_t(i) \phi( f_i(x_t) )$, where $\calF$ contains score functions
$f = (f_1, \ldots, f_K)$ such that $\sum_{i=1}^K f_i(\cdot) \equiv 0$, and $\phi(s) = \max(1 - \frac
s \gamma, 0)$ or $\min(1, \max(1 - \frac s \gamma, 0))$. On one hand, it gives
information-theoretic regret upper bounds for various settings of $\calF$. On
the other hand, it gives an efficient algorithm with an $O(\sqrt{T})$ regret
against the benchmark of $\calF = \cbr{x \mapsto W x: W \in \R^{K \times d},
\one^T W = 0}$. A direct application of this result to \textsc{Online Bandit
Multiclass Linear Classification} gives an algorithm with $O(\sqrt{T})$ mistake
bound in the strongly linearly separable case.

%including parametric and nonparametric classes

%\cite{Chen-Chen-Zhang-Chen-Zhang-2009} studies the approach of reducing bandit
%multiclass learning to online binary classification using one-versus-all
%reduction. They show that

%Their results do not imply a finite mistake bound in the weakly separable
%setting, in that the benchmark loss can still be $\Omega(T)$.

\section{Notions of linear separability}
\label{section:notions-of-linear-separability}
Let $[n]=\{1,2,\ldots,n\}$.
We define two notions of linear separability for multiclass classification. The
first notion is the standard notion of linear separability used in the proof of
the mistake bound for the \textsc{Multiclass Perceptron} algorithm \citep[see e.g.][]{Crammer-Singer-2003}. The second
notion is stronger, i.e. more restrictive. %However, it is more suitable for the
%bandit setting, since it allows for a simple and efficient algorithm; see
%Section~\ref{section:algorithm-for-strongly-linearly-separable-data}.

\begin{definition}
[Linear separability]
\label{definition:linear-separability}
Let $(V,\ip{\cdot}{\cdot})$ be an inner product space, $K$ be a positive
integer, and $\gamma$ be a positive real number.
We say that labeled examples $(x_1, y_1),
(x_2, y_2), \dots, (x_T, y_T) \in V \times [K]$ are
\\ \vskip .01in
\emph{weakly linearly separable with a margin $\gamma$} if there exist vectors
$w_1, w_2, \dots, w_K \in V$ such that
\begin{align}
\label{equation:weak-linear-separability-1}
\sum_{i=1}^K \norm{w_i}^2 & \le 1, \\
\label{equation:weak-linear-separability-2}
\ip{x_t}{w_{y_t}} & \ge \ip{x_t}{w_i} + \gamma \quad
\forall t \in [T]\ \forall i \in [K] \setminus \{y_t\}, 
\end{align}

and \emph{strongly linearly separable with a margin $\gamma$} if there exist vectors
$w_1, w_2, \dots, w_K \in V$ such that
\begin{align}
\label{equation:strong-linear-separability-1}
\sum_{i=1}^K \norm{w_i}^2 & \le 1, \\
\label{equation:strong-linear-separability-2}
\ip{x_t}{w_{y_t}} & \ge \gamma/2 \quad  \forall t \in [T], \\
\label{equation:strong-linear-separability-3}
\ip{x_t}{w_i} & \le - \gamma/2 \quad
\forall t \in [T]\ \forall i \in [K] \setminus \{y_t\}.
\end{align}
\end{definition}

%The notion of weak linear separability is a standard assumption used in the
%full-information setting to upper bound the number of mistakes of the
%\textsc{Multiclass Perceptron} algorithm.
%%%%%
%(Recall our discussion on the topic in Section~\ref{section:introduction}.)
%Namely, \citet{Crammer-Singer-2003}
%prove that if the examples are weakly linearly separable with a margin $\gamma$
%and the norm of the examples is bounded by $R$, then \textsc{Multiclass
%Perceptron} algorithm makes at most $\left\lfloor 2(R/\gamma)^2 \right \rfloor$
%mistakes. It is a folklore result that \textsc{Multiclass Perceptron} is
%essentially optimal in the sense that any (possibly randomized) algorithm must
%make $\frac{1}{2} \left\lfloor (R/\gamma)^2 \right \rfloor$ mistakes in the
%worst case. For completeness, we state the \textsc{Multiclass Perceptron}
%algorithm, the upper and lower bound on the number of mistakes, and their proofs
%in Appendix~\ref{section:multiclass-perceptron-proofs} in the supplementary
%material.
%%%%%

The notion of strong linear separability has appeared in the literature; see
e.g.~\citep{Chen-Chen-Zhang-Chen-Zhang-2009}. Intuitively, strong linear
separability means that, for each class $i$, the set of examples belonging to
class $i$ and the set of examples belonging to the remaining $K-1$ classes are
separated by a linear classifier $w_i$ with margin $\frac{\gamma}{2}$.

It is easy to see that if a set of labeled examples is strongly linearly
separable with margin $\gamma$, then it is also weakly linearly separable with
the same margin (or larger). Indeed, if $w_1, w_2, \dots, w_K \in V$ satisfy
\eqref{equation:strong-linear-separability-1},
\eqref{equation:strong-linear-separability-2},
\eqref{equation:strong-linear-separability-3} then they satisfy
\eqref{equation:weak-linear-separability-1} and
\eqref{equation:weak-linear-separability-2}.

In the special case of $K=2$, if a set of labeled examples is weakly
linearly separable with a margin $\gamma$, then it is also strongly linearly
separable with the same margin. Indeed, if $w_1, w_2$ satisfy
\eqref{equation:weak-linear-separability-1} and
\eqref{equation:weak-linear-separability-2} then $w_1' = \frac{w_1 - w_2}{2}$,
$w_2' = \frac{w_2 - w_1}{2}$ satisfy
\eqref{equation:strong-linear-separability-1},
\eqref{equation:strong-linear-separability-2},
\eqref{equation:strong-linear-separability-3}.
Equation~\eqref{equation:strong-linear-separability-1} follows from
$\norm{w_i'}^2 \le (\frac{1}{2} \norm{w_1} + \frac{1}{2} \norm{w_2})^2 \le
\frac{1}{2}\norm{w_1}^2 + \frac{1}{2}\norm{w_2}^2 \le \frac{1}{2}$ for $i=1,2$.
Equations~\eqref{equation:strong-linear-separability-2}
and~\eqref{equation:strong-linear-separability-3} follow from the fact that
$w_1' - w_2' = w_1 - w_2$.

However, for any $K \ge 3$ and any inner product space of dimension at least
$2$, there exists a set of labeled examples that is weakly linearly separable
with a positive margin $\gamma$ but is not strongly linearly separable with
any positive margin.~\autoref{figure:weakly-linearly-separable-examples-with-margin}
shows one such set of labeled examples.

\begin{figure}
\begin{center}
 \scalebox{.8}{\begin{tikzpicture}[scale=0.4]

  \useasboundingbox (-6.5,-6) rectangle (12.5,6);
  % \draw[help lines] (-5.5,-5.5) rectangle (5.5,5.5);

  % Grid and coordinate axes
  %% \draw[help lines] (-5.2,-5.2) grid (5.2,5.2);
  %% \draw[thick, ->, >=latex] (-5.2,0) -- (5.2,0);
  %% \draw[thick, ->,  >=latex] (0,-5.2) -- (0,5.2);

  % Notable points
  \coordinate (Origin) at (0,0);

  \coordinate [label={[xshift=-5mm, yshift=2.5mm]$\ip{w_1 - w_2}{x} = 0$}] (A) at (120:5);
  \coordinate [label={[xshift=-5mm, yshift=-10mm]$\ip{w_2 - w_3}{x} = 0$}] (B) at (240:5);
  \coordinate [label={[xshift=+17mm, yshift=-1.5mm]$\ip{w_3 - w_1}{x} = 0$}] (C) at (360:5);

  % Circle which contains all the examples.
  \draw (Origin) circle (5);

  \draw[ultra thick] (Origin) -- (A);
  \draw[ultra thick] (Origin) -- (B);
  \draw[ultra thick] (Origin) -- (C);

  % Class 0 points. Angles between 0.0 and 120.0
  \draw[color=black, fill=none]
                    (7.94:3.79) circle (4pt)
                    (10.97:4.51) circle (4pt)
                    (18.63:2.30) circle (4pt)
                    (21.20:4.42) circle (4pt)
                    (26.23:3.57) circle (4pt)
                    (31.66:4.63) circle (4pt)
                    (41.52:4.77) circle (4pt)
                    (42.22:2.06) circle (4pt)
                    (44.90:3.30) circle (4pt)
                    (52.95:4.01) circle (4pt)
                    (57.79:2.62) circle (4pt)
                    (58.25:4.62) circle (4pt)
                    (65.64:3.86) circle (4pt)
                    (68.11:1.10) circle (4pt)
                    (71.38:3.00) circle (4pt)
                    (72.60:1.97) circle (4pt)
                    (77.86:4.23) circle (4pt)
                    (87.11:4.44) circle (4pt)
                    (92.77:2.90) circle (4pt)
                    (93.88:2.09) circle (4pt)
                    (93.97:3.90) circle (4pt)
                    (100.87:4.39) circle (4pt)
                    (105.50:3.40) circle (4pt)
                    (107.58:2.57) circle (4pt)
                    (110.09:4.31) circle (4pt)
                    ;

  % Class 1 points. Angles between 120.0 and 240.0
  \draw[color=gray, fill=gray]
                    (130.71:4.76) circle (4pt)
                    (132.80:3.70) circle (4pt)
                    (133.71:2.84) circle (4pt)
                    (141.55:4.39) circle (4pt)
                    (144.89:2.14) circle (4pt)
                    (144.96:3.67) circle (4pt)
                    (150.11:2.92) circle (4pt)
                    (152.47:4.79) circle (4pt)
                    (157.09:0.90) circle (4pt)
                    (160.45:4.13) circle (4pt)
                    (162.92:1.80) circle (4pt)
                    (164.13:2.68) circle (4pt)
                    (170.32:4.44) circle (4pt)
                    (174.55:3.69) circle (4pt)
                    (179.08:3.01) circle (4pt)
                    (179.39:4.49) circle (4pt)
                    (182.94:2.24) circle (4pt)
                    (193.19:4.09) circle (4pt)
                    (196.53:3.29) circle (4pt)
                    (197.93:1.08) circle (4pt)
                    (198.39:2.57) circle (4pt)
                    (208.08:1.93) circle (4pt)
                    (209.68:4.07) circle (4pt)
                    (210.92:3.16) circle (4pt)
                    (217.72:4.79) circle (4pt)
                    (222.35:3.61) circle (4pt)
                    (225.84:4.36) circle (4pt)
                    (227.89:2.91) circle (4pt)
                    ;

% Class 2 points. Angles between 240.0 and 360.0
  \draw[color=black, fill=black]
                    (248.23:3.89) circle (4pt)
                    (251.03:4.78) circle (4pt)
                    (253.86:2.55) circle (4pt)
                    (258.82:3.67) circle (4pt)
                    (270.54:2.70) circle (4pt)
                    (271.79:3.46) circle (4pt)
                    (272.88:2.00) circle (4pt)
                    (274.04:4.62) circle (4pt)
                    (282.58:4.75) circle (4pt)
                    (283.72:3.04) circle (4pt)
                    (286.97:1.29) circle (4pt)
                    (288.41:3.80) circle (4pt)
                    (293.24:4.78) circle (4pt)
                    (297.90:2.43) circle (4pt)
                    (299.36:4.22) circle (4pt)
                    (304.20:3.46) circle (4pt)
                    (308.83:4.46) circle (4pt)
                    (313.94:2.13) circle (4pt)
                    (319.47:4.33) circle (4pt)
                    (324.13:2.82) circle (4pt)
                    (335.46:3.98) circle (4pt)
                    (337.60:1.75) circle (4pt)
                    (342.13:2.45) circle (4pt)
                    (345.92:3.32) circle (4pt)
                    (349.18:4.53) circle (4pt)
                    ;

\end{tikzpicture}}
\end{center}
\caption[]{A set of labeled examples in $\R^2$. The examples belong to
$K=3$ classes colored white, gray and black respectively. Each class lies in a
$120^\circ$ wedge. In other words, each class lies in an intersection of two
halfspaces. While the examples are weakly linearly separable with a positive margin
$\gamma$, they are \emph{not} strongly linearly separable with any positive
margin $\gamma$. For instance, there does \emph{not} exist a linear separator
that separates the examples belonging to the gray class from the examples
belonging to the remaining two classes.
}
\label{figure:weakly-linearly-separable-examples-with-margin}
\end{figure}
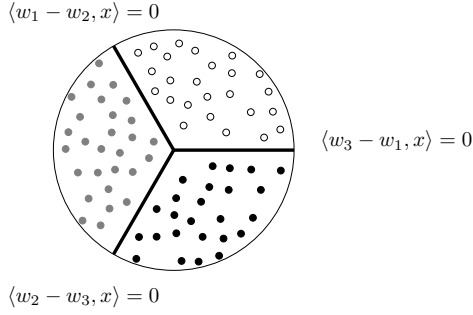

\section{Algorithm for strongly linearly separable data}
\label{section:algorithm-for-strongly-linearly-separable-data}

In this section, we consider the case when the examples are strongly linearly
separable. We present an algorithm for this setting
(Algorithm~\ref{algorithm:algorithm-for-strongly-linearly-separable-examples})
and give an upper bound on its number of mistakes, stated as
\autoref{theorem:strongly-separable-examples-mistake-upper-bound} below. The
proof of the theorem can be found in
Appendix~\ref{section:proofs-for-stringly-separable-examples}.

The idea behind
Algorithm~\ref{algorithm:algorithm-for-strongly-linearly-separable-examples} is
to use $K$ copies of the \textsc{Binary Perceptron} algorithm, one copy per
class; see e.g. \citep[Section 3.3.1]{Shalev-Shwartz-2012}. Upon seeing each
example $x_t$, copy $i$ predicts whether or not $x_t$ belongs to class $i$.
Multiclass predictions are done by evaluating all $K$ binary predictors and
outputting any class with a positive prediction.  If all binary predictions are
negative, the algorithm chooses a prediction uniformly at random
from $\cbr{1,2,\dots,K}$.

%\begin{algorithm}[h]
%\caption{\textsc{Bandit Algorithm for Strongly Linearly Separable Examples}
%\label{algorithm:algorithm-for-strongly-linearly-separable-examples}}
%\begin{algorithmic}[1]
%{
%\REQUIRE{Number of classes $K$, number of rounds $T$.}
%\REQUIRE{Inner product space $(V,\ip{\cdot}{\cdot})$.}
%\STATE{Initialize $w_1^{(1)} = w_2^{(1)} = \dots = w_K^{(1)} = 0$} \\
%\FOR{$t=1,2,\dots,T$}
%\STATE{Observe feature vector $x_t \in V$}
%\STATE{Compute $S_t = \left\{ i ~:~ 1 \le i \le K, \ \ip{w_i^{(t)}}{x_t} \ge 0 \right\}$}
%\IF{$S_t = \emptyset$}
%\STATE{Predict $\widehat y_t \sim \text{Uniform}(\{1,2,\dots,K\})$}
%\STATE{Observe feedback $z_t = \indicator{\widehat y_t \neq y_t}$}
%\IF{$z_t = 1$}
%\STATE{Set $w_i^{(t+1)} = w_i^{(t)}$ for all $i \in \{1,2,\dots,K\}$}
%\ELSE
%\STATE{Set $w_i^{(t+1)} = w_i^{(t)}$ \\ \qquad  for all $i \in \{1,2,\dots,K\} \setminus \{\widehat y_t\}$}
%\STATE{Update $w_{\widehat y_t}^{(t+1)} = w_{\widehat y_t}^{(t)} + x_t$}
%\label{line:pos-update}
%\ENDIF
%\ELSE
%\STATE{Predict $\widehat y_t \in S_t$ chosen arbitrarily}
%\STATE{Observe feedback $z_t  = \indicator{\widehat y_t \neq y_t}$}
%\IF{$z_t = 1$}
%\STATE{Set $w_i^{(t+1)} = w_i^{(t)}$ \\ \qquad for all $i \in \{1,2,\dots,K\} \setminus \{\widehat y_t\}$}
%\STATE{Update $w_{\widehat y_t}^{(t+1)} = w_{\widehat y_t}^{(t)} - x_t$}
%\label{line:neg-update}
%\ELSE
%\STATE{Set $w_i^{(t+1)} = w_i^{(t)}$ for all $i \in \{1,2,\dots,K\}$}
%\ENDIF
%\ENDIF
%\ENDFOR
%}
%\end{algorithmic}
%\end{algorithm}

\begin{algorithm}[h]
\SetAlgoLined
\LinesNumbered
\caption{\textsc{Bandit Algorithm for Strongly Linearly Separable Examples}
\label{algorithm:algorithm-for-strongly-linearly-separable-examples}}
\textbf{Require:} Number of classes $K$, number of rounds $T$. \\
\textbf{Require:} Inner product space $(V,\ip{\cdot}{\cdot})$.  \\
\nl Initialize $w_1^{(1)} = w_2^{(1)} = \dots = w_K^{(1)} = 0$\\
\nl \For{$t=1,2,\dots,T$}{
\nl    Observe feature vector $x_t \in V$ \\
\nl    Compute $S_t = \left\{ i ~:~ 1 \le i \le K, \ \ip{w_i^{(t)}}{x_t} \ge 0 \right\}$\\
\nl    \If{$S_t = \emptyset$}{
\nl         Predict $\widehat y_t \sim \text{Uniform}(\{1,2,\dots,K\})$ \\
\nl         Observe feedback $z_t = \indicator{\widehat y_t \neq y_t}$\\
\nl          \If{$z_t = 1$}{
\nl               Set $w_i^{(t+1)} = w_i^{(t)}$, $\forall i \in \{1,2,\dots,K\}$
              }
\nl          \Else{
\nl               Set $w_i^{(t+1)} = w_i^{(t)}$, $\forall i \in \{1,2,\dots,K\} \setminus \{\widehat y_t\}$ \\
\nl Update $w_{\widehat y_t}^{(t+1)} = w_{\widehat y_t}^{(t)} + x_t$  \label{line:pos-update}
              }
        }
\nl     \Else{
\nl           Predict $\widehat y_t \in S_t$ chosen arbitrarily \\
\nl           Observe feedback $z_t  = \indicator{\widehat y_t \neq y_t}$ \\
\nl            \If{$z_t = 1$}{
\nl                Set $w_i^{(t+1)} = w_i^{(t)}$, $\forall i \in \{1,2,\dots,K\} \setminus \{\widehat y_t\}$ \\
\nl  Update $w_{\widehat y_t}^{(t+1)} = w_{\widehat y_t}^{(t)} - x_t$   \label{line:neg-update}
               }
\nl           \Else{
\nl                 Set $w_i^{(t+1)} = w_i^{(t)}$, $\forall i \in \{1,2,\dots,K\}$
           }
      }
}
\end{algorithm}

\begin{theorem}[Mistake upper bound]
\label{theorem:strongly-separable-examples-mistake-upper-bound}
Let $(V, \ip{\cdot}{\cdot})$ be an inner product space, $K$ be a positive
integer, $\gamma$ be a positive real number, $R$ be a non-negative real number.
If the examples $(x_1, y_1), \dots, (x_T, y_T) \in V \times \{1,2,\dots,K\}$ are
strongly linearly separable with margin $\gamma$ and $\norm{x_1}, \norm{x_2},
\dots, \norm{x_T} \le R$ then the expected number of mistakes that
Algorithm~\ref{algorithm:algorithm-for-strongly-linearly-separable-examples}
makes is at most $(K-1) \lfloor 4(R/\gamma)^2 \rfloor$.
\end{theorem}

The upper bound $(K-1) \lfloor 4(R/\gamma)^2 \rfloor$ on the expected number of
mistakes of
Algorithm~\ref{algorithm:algorithm-for-strongly-linearly-separable-examples} is
optimal up to a constant factor, as long as the number of classes $K$ is at most
$O((R/\gamma)^2)$. This lower bound is stated as
\autoref{theorem:strongly-separable-examples-mistake-lower-bound} below. The
proof of the theorem can be found in
Appendix~\ref{section:proofs-for-stringly-separable-examples}.
\citet{Daniely-Helbertal-2013} provide a lower bound
under the assumption of weak linear separability, which does not immediately
imply a lower bound under the stronger notion.

\begin{theorem}[Mistake lower bound]
\label{theorem:strongly-separable-examples-mistake-lower-bound}
Let $\gamma$ be a positive real number, $R$ be a non-negative real number and
let $K \le (R/\gamma)^2$ be a positive integer. Any (possibly randomized)
algorithm makes at least $((K-1)/2)\left\lfloor (R/\gamma)^2/4 \right\rfloor$
mistakes in expectation on some sequence of labeled examples $(x_1, y_1),
(x_2, y_2), \dots, (x_T, y_T) \in V \times \{1,2,\dots,K\}$ for some inner
product space $(V, \ip{\cdot}{\cdot})$ such that the examples are strongly
linearly separable with margin $\gamma$ and satisfy $\norm{x_1}, \norm{x_2},
\dots, \norm{x_T} \le R$.
\end{theorem}

\paragraph{Remark.}
If $\gamma \le R$ then, irrespective of any other conditions on $K$, $R$, and
$\gamma$, a trivial lower bound on the expected number of mistakes of any
randomized algorithm is $(K-1)/2$. To see this, note that the adversary can
choose an example $(R e_1, y)$, where $e_1$ is some arbitrary unit vector in $V$
and $y$ is a label chosen uniformly from $\cbr{1,2,\dots,K}$, and show this
example $K$ times. The sequence of examples trivially satisfies the strong
linear separability condition, and the $(K-1)/2$ expected mistake lower bound
follows from ~\citep[][Claim 2]{Daniely-Helbertal-2013}. %However, it is unclear
%if the trivial lower bound is the best possible if $K$ is large, i.e.,
%$\omega((R/\gamma)^2)$. We leave this as an open problem.

Algorithm~\ref{algorithm:algorithm-for-strongly-linearly-separable-examples} can
be extended to nonlinear classification using \emph{positive definite kernels}
(or \emph{kernels}, for short), which are functions of the form $k: X \times X
\to \R$ for some set $X$ such that the matrix
$\left(k(x_i,x_j)\right)_{i,j=1}^m$ is a symmetric positive semidefinite for any
positive integer $m$ and $x_1, x_2, \dots, x_m \in X$ \citep[Definition
2.5]{Scholkopf-Smola-2002}.\footnote{For every kernel there exists an associated
feature map $\phi:X \to V$ into some inner product space $(V,\ip{\cdot}{\cdot})$
such that $k(x,x') = \ip{\phi(x)}{\phi(x')}$.} As opposed to explicitly
maintaining the weight vector for each class, the algorithm maintains the set of
example-scalar pairs corresponding to the updates of the non-kernelized
algorithm. As a direct consequence of
Theorem~\ref{theorem:strongly-separable-examples-mistake-upper-bound} we get a
mistake bound for the kernelized algorithm.

\begin{theorem}[Mistake upper bound for kernelized algorithm]
\label{theorem:kernelized-upper-bound}
Let $X$ be a non-empty set, let $(V, \ip{\cdot}{\cdot})$ be an inner product
space. Let $\phi:X \to V$ be a feature map and let $k:X \times X \to \R$,
$k(x,x') = \ip{\phi(x)}{\phi(x')}$ be the associated positive definite
kernel. Let $K$ be a positive integer, $\gamma$ be a positive real number, $R$
be a non-negative real number. If $(x_1, y_1), (x_2, y_2), \dots, (x_T, y_T) \in
X \times \{1,2,\dots,K\}$ are labeled examples such that:
\begin{enumerate}
\item the mapped examples $(\phi(x_1), y_1)$, $\dots$, $(\phi(x_T), y_T)$
are strongly linearly separable with margin $\gamma$,
\item $k(x_1, x_1), k(x_2, x_2), \dots, k(x_T,x_T) \le R^2$,
\end{enumerate}
then the expected number of mistakes that Algorithm~\ref{algorithm:kernelized} makes
is at most $(K-1) \lfloor 4(R/\gamma)^2 \rfloor$.
\end{theorem}

\begin{algorithm}[h]
\caption{\textsc{Kernelized Bandit Algorithm}
\label{algorithm:kernelized}}
\textbf{Require:} Number of classes $K$, number of rounds $T$.\\
\textbf{Require:} Kernel function  $k(\cdot, \cdot)$. \\
Initialize $J_1^{(1)} = J_2^{(1)} = \dots = J_K^{(1)} = \emptyset$ \\
\For{$t=1,2,\dots,T$} {
     Observe feature vector $x_t$.\\
     Compute \\ $S_t = \left\{ i ~:~ 1 \le i \le K, \ \sum_{(x,y) \in J_i^{(t)}} y k(x, x_t) \ge 0 \right\}$ \\
     \If{$S_t = \emptyset$}{
           Predict $\widehat y_t \sim \text{Uniform}(\{1,2,\dots,K\})$ \\
           Observe feedback $z_t = \indicator{\widehat y_t \neq y_t}$ \\
           \If{$z_t = 1$}{
                Set $J_i^{(t+1)} = J_i^{(t)}$ for all $i \in \{1,2,\dots,K\}$
           }
           \Else{
                Set $J_i^{(t+1)} = J_i^{(t)}$, $\forall i \in \{1,2,\dots,K\} \setminus \{\widehat y_t\}$ \\
                Update $J_{\widehat y_t}^{(t+1)} = J_{\widehat y_t}^{(t)} \cup \cbr{(x_t, +1)}$
           }
      }
      \Else{
           Predict $\widehat y_t \in S_t$ chosen arbitrarily \\
           Observe feedback $z_t =  \indicator{\widehat y_t \neq y_t}$ \\
           \If{$z_t = 1$}{
                Set $J_i^{(t+1)} = J_i^{(t)}$, $\forall i \in \{1,2,\dots,K\} \setminus \{\widehat y_t\}$\\
                Update $J_{\widehat y_t}^{(t+1)} = J_{\widehat y_t}^{(t)} \cup \cbr{(x_t, -1)}$
           }
           \Else {
                Set $J_i^{(t+1)} = J_i^{(t)}$ for all $i \in \{1,2,\dots,K\}$
           }
      }
}

\end{algorithm}

\section{From weak separability to strong separability}
\label{section:from-weak-separability-to-strong-separability}

In this section, we consider the case when the examples are weakly linearly
separable. Throughout this section, we assume without loss of generality that
all examples lie in the unit ball $\B(\zero,1) \subseteq \R^d$.\footnote{Instead of
working with feature vector $x_t$ we can work with normalized feature vectors
$\widehat{x}_t = \frac{x_t}{\norm{x_t}}$. It can be easily checked that if
$(x_1,y_1), (x_2,y_2), \dots, (x_T,y_T)$ are weakly linearly separable with
margin $\gamma$ and $\norm{x_t} \le R$ for all $t$, then the normalized examples
$(\widehat{x}_1,y_1), (\widehat{x}_2,y_2), \dots, (\widehat{x}_T,y_T)$ are weakly
linearly separable with margin $\gamma/R$.} Note that
Algorithm~\ref{algorithm:algorithm-for-strongly-linearly-separable-examples}
alone does not guarantee a finite mistake bound in this setting, as weak linear
separability does not imply strong linear separability.

We use a positive definite kernel function $k(\cdot, \cdot)$, namely
a \emph{rational kernel}~\citep{Shalev-Shwartz-Shamir-Sridharan-2011} whose
corresponding feature map $\phi(\cdot)$ transforms any sequence of \emph{weakly}
linearly separable examples to a \emph{strongly} linearly separable sequence of
examples.
Specifically, $\phi$ has the property that if a set of labeled
examples in $\B(\zero,1)$ is weakly linearly separable with a margin $\gamma$, then
after applying $\phi$ the examples become strongly linearly separable with a
margin $\gamma'$ and their squared norms are bounded by $2$.
\footnote{Other kernels, such as the polynomial kernel
$k(x,x') = (1+\ip{x}{x'})^d$, or the multinomial kernel~\citep{Goel-Klivans-2017} $k(x,x') = \sum_{i=0}^d(\ip{x}{x'})^i$,
will have similar properties for large enough $d$.}
The parameter $\gamma'$ is
a function of the old margin $\gamma$ and the number of classes $K$, and is
specified in \autoref{theorem:margin-transformation} below.

The rational kernel $k:\B(\zero,1) \times \B(\zero,1) \to \R$ is defined as
\begin{equation}
\label{equation:rational-kernel}
k(x,x') = \frac{1}{1 - \frac{1}{2}\ip{x}{x'}_{\R^d}} \; .
\end{equation}
Note that $k(x,x')$ can be evaluated in $O(d)$ time.

Consider the classical real separable Hilbert space $\ell_2 = \{ x \in \R^\infty
~:~ \sum_{i=1}^\infty x_i^2 < + \infty \}$ equipped with the standard inner
product $\ip{x}{x'}_{\ell_2} = \sum_{i=1}^\infty x_i x'_i$. If we index the
coordinates of $\ell_2$ by $d$-tuples $(\alpha_1, \alpha_2, \dots, \alpha_d)$ of
non-negative integers, the feature map that corresponds to $k$ is $\phi: \B(\zero,1)
\to \ell_2$,
\begin{align}
\begin{split}
&\left(\phi(x_1, x_2, \dots, x_d)\right)_{(\alpha_1, \alpha_2, \dots, \alpha_d)} = \\
& x_1^{\alpha_1} x_2^{\alpha_2} \dots x_d^{\alpha_d} \cdot \sqrt{2^{-(\alpha_1 + \alpha_2 + \dots + \alpha_d)} \binom{\alpha_1 + \alpha_2 + \dots + \alpha_d}{\alpha_1, \alpha_2, \dots, \alpha_d}}
\end{split}
\label{equation:phi}
\end{align}
where $\binom{\alpha_1 + \alpha_2 + \dots + \alpha_d}{\alpha_1, \alpha_2, \dots,
\alpha_d} = \frac{(\alpha_1 + \alpha_2 + \dots + \alpha_d)!}{\alpha_1! \alpha_2! \dots \alpha_d!}$ is the multinomial coefficient. It can be easily checked that
$$
k(x,x') = \ip{\phi(x)}{\phi(x')}_{\ell_2}.
$$
The last equality together with the formula for $k$ implies that $k(x,x) <
+\infty$ for any $x$ in $\B(\zero,1)$ and thus in particular implies that $\phi(x)$
indeed lies in $\ell_2$.

The following theorem is our main technical result in this section. We defer its
proof to Section~\ref{section:margin-transformation}.

\begin{theorem}[Margin transformation]
\label{theorem:margin-transformation}
Let $(x_1, y_1)$, $(x_2, y_2)$, $\dots$, $(x_T, y_T)$ $\in \B(\zero,1) \times
\{1,2,\dots,K\}$ be a sequence of labeled examples that is weakly linearly
separable with margin $\gamma > 0$. Let $\phi$ be as defined in
equation~\eqref{equation:phi} and let
%\begingroup
%\allowdisplaybreaks
%\begin{align*}
%r & = 2 \left\lceil \frac{1}{4} \log_2(4K-3) \right\rceil + 1, \quad \quad s = \left \lceil \log_2(2/\gamma) \right \rceil, \\
%\gamma_1 & = \frac{1}{2\sqrt{K}}  \\
%& \ \cdot \left(376 \lceil \log_2(2K-2) \rceil \cdot \left \lceil \sqrt{\frac{2}{\gamma}} \right \rceil \right)^{-\frac{1}{2} \lceil \log_2(2K-2) \rceil \cdot \left \lceil \sqrt{\frac{2}{\gamma}} \right \rceil}, \\
%\gamma_2 & = \frac{2^{s(s+1)r(K-1)} }{4\sqrt{K}(4K-5) 2^{K-1}} \\
%& \ \cdot \left(2^{s+1} r(K-1) (4s+2)^2 \right)^{-(s+1/2)r(K-1)} \; .
%\end{align*}
%\endgroup
\begingroup
\allowdisplaybreaks
\begin{align*}
%r & = 2 \left\lceil \frac{1}{4} \log_2(4K-3) \right\rceil + 1, \quad \quad s = \left \lceil \log_2(2/\gamma) \right \rceil, \\
\gamma_1 = & \frac{
  \left[
    376 \lceil \log_2(2K-2) \rceil \cdot \left \lceil \sqrt{\frac{2}{\gamma}} \right \rceil
  \right]^{
    \frac{-\lceil \log_2(2K-2) \rceil \cdot \left \lceil \sqrt{{2}/{\gamma}} \right \rceil}{2}
  }
}{2\sqrt{K}},\\
\gamma_2 = & \frac{
   \left(2^{s+1} r(K-1) (4s+2) \right)^{-(s+1/2)r(K-1)}
}{4\sqrt{K}(4K-5) 2^{K-1}}
\; ,
\end{align*}
\endgroup

%\\
%&  \cdot 2^{s(s+1)r(K-1)}

where $r = 2 \left\lceil \frac{1}{4} \log_2(4K-3) \right\rceil + 1$ and $s = \left \lceil \log_2(2/\gamma) \right \rceil$.
Then, the sequence of labeled examples transformed by $\phi$,
namely $(\phi(x_1), y_1), (\phi(x_2), y_2), \dots,
(\phi(x_T), y_T)$, is strongly linearly separable with margin $\gamma' =
\max\{\gamma_1, \gamma_2\}$. In addition, for all $t$ in $\cbr{1,\ldots,T}$,
$k(x_t,x_t) \leq 2$.
\end{theorem}

Using this theorem we derive a mistake bound for
Algorithm~\ref{algorithm:kernelized} with kernel \eqref{equation:rational-kernel}
under the weak linear separability assumption.

\begin{corollary}[Mistake upper bound]
\label{corollary:weakly-separable-examples-mistake-upper-bound}
Let $K$ be a positive integer and let $\gamma$ be a positive real number. If
$(x_1, y_1), (x_2, y_2), \dots, (x_T, y_T) \in \B(\zero,1) \times \{1,2,\dots,K\}$
is a sequence of weakly separable labeled examples with margin $\gamma > 0$,
then the expected number of mistakes made by Algorithm~\ref{algorithm:kernelized}
with kernel $k(x,x')$ defined by \eqref{equation:rational-kernel}
is at most $\min (2^{\widetilde{O}(K \log^2
(1/\gamma))}, 2^{\widetilde{O}(\sqrt{1/\gamma} \log K)})$.
\end{corollary}

This corollary follows directly from
Theorems~\ref{theorem:kernelized-upper-bound}
and~\ref{theorem:margin-transformation}. We remark that under the weakly linearly
separable setting,~\citep{Daniely-Helbertal-2013} gives a mistake lower bound of
$\Omega(\frac{K}{\gamma^2})$ for {\em any algorithm} (see also~\autoref{theorem:strongly-separable-examples-mistake-lower-bound}).
We leave the possibility of designing efficient algorithms that
have mistakes bounds matching this lower bound as an important open question.

\subsection{Proof of Theorem~\ref{theorem:margin-transformation}}
\label{section:margin-transformation}

\paragraph{Overview.} The idea behind the construction and analysis of the
mapping $\phi$ is polynomial approximation. Specifically, we construct $K$
multivariate polynomials $p_1, p_2, \dots,p_K$ such that
\begin{align}
\label{equation:poly-pos}
& \forall t \in \cbr{1,2,\dots,T}, \qquad p_{y_t}(x_t) \ge \frac{\gamma'}{2} \; ,
\\
\label{equation:poly-neg}
& \begin{gathered}
\forall t \in \cbr{1,2,\dots,T} \ \forall i \in \cbr{1,2,\ldots,K} \setminus \cbr{y_t}, \\
p_i(x_t) \le - \frac{\gamma'}{2} \; .
\end{gathered}
\end{align}
We then show (\autoref{lemma:norm-bound}) that each polynomial $p_i$ can be
expressed as $\ip{c_i}{\phi(x)}_{\ell_2}$ for some $c_i \in \ell_2$. This
immediately implies that the examples $(\phi(x_1),y_1), \ldots,
(\phi(x_T),y_T)$ are strongly linearly separable with a positive margin.

The conditions \eqref{equation:poly-pos} and~\eqref{equation:poly-neg} are
equivalent to that
\begin{align}
\label{equation:poly-pos-i}
\forall t \in \cbr{1,2,\dots,T}, y_t = i \quad \Rightarrow \quad p_i(x_t) \ge \frac{\gamma'}{2} \; , \\
\label{equation:poly-neg-i}
\forall t \in \cbr{1,2,\dots,T}, y_t \neq i \quad \Rightarrow \quad p_i(x_t) \le -\frac{\gamma'}{2} \; .
\end{align}
hold for all $i \in \{1,2,\dots,K\}$. We can thus fix $i$ and focus on
construction of one particular polynomial $p_i$.

Since examples $(x_1,y_1), (x_2, y_2), \dots, (x_T,y_T)$ are weakly linearly separable,
all examples from class $i$ lie in
$$
R_i^+ = \!\!\!\!\! \!\!\!\!\! \bigcap_{j \in \cbr{1,2,\dots,K} \setminus \cbr{i}} \cbr{x \in \B(\zero,1) ~:~ \ip{w_i^* - w_j^*}{x} \ge \gamma},
$$
and all examples from the remaining classes lie in
$$
R_i^- = \!\!\!\!\! \!\!\!\!\! \bigcup_{j \in \cbr{1,2,\dots,K} \setminus \cbr{i}} \cbr{x \in \B(\zero,1) ~:~ \ip{w_i^* - w_j^*}{x} \le -\gamma}.
$$
Therefore, to satisfy conditions~\eqref{equation:poly-pos-i}
and~\eqref{equation:poly-neg-i}, it suffices to construct $p_i$ such that
\begin{align}
\label{eqn:r-plus}
x \in R_i^+ \qquad & \Longrightarrow \qquad p_i(x) \ge \frac {\gamma'} 2 \; , \\
\label{eqn:r-minus}
x \in R_i^- \qquad & \Longrightarrow \qquad p_i(x) \le - \frac {\gamma'} 2 \; .
\end{align}

According to the well known Stone-Weierstrass theorem~\citep[see
e.g.][Section~10.10]{Davidson-Donsig-2010}, on a compact set, multivariate
polynomials uniformly approximate any continuous function. Roughly speaking, the
conditions \eqref{eqn:r-plus} and \eqref{eqn:r-minus} mean that $p_i$
approximates on $\B(\zero,1)$ a scalar multiple of the indicator function of the
intersection of $K-1$ halfspaces $\bigcap_{j \in \cbr{1,2,\dots,K} \setminus
\cbr{i}} \cbr{x: \ip{w_i^* - w_j^*}{x} \geq 0}$ while within margin $\gamma$ along
the decision boundary, the polynomial is allowed to attain arbitrary values.
It is thus clear such a polynomial exists.

We give two explicit constructions for such polynomial in
Theorems~\ref{theorem:polynomial-approximation-1}~and~\ref{theorem:polynomial-approximation-2}.
Our constructions are based on \citet{Klivans-Servedio-2008} which in turn uses
the constructions from~\citet{Beigel-Reingold-Spielman-1995}. More importantly,
the theorems quantify certain parameters of the polynomial, which allows us
to upper bound the transformed margin $\gamma'$.

Before we state the theorems, recall that a polynomial of $d$ variables is a
function $p:\R^d \to \R$ of the form
\begin{align*}
p(x)
&= p(x_1, x_2, \dots, x_d) \\
&= \sum_{\alpha_1, \alpha_2, \dots, \alpha_d} c_{\alpha_1, \alpha_2, \dots, \alpha_d} x_1^{\alpha_1} x_2^{\alpha_2} \dots x_d^{\alpha_d}
\end{align*}
where the sum ranges over a finite set of $d$-tuples $(\alpha_1, \alpha_2,
\dots, \alpha_d)$ of non-negative integers and $c_{\alpha_1, \alpha_2, \dots,
\alpha_d}$'s are real coefficients. The \emph{degree} of a polynomial $p$,
denoted by $\deg(p)$, is the largest value of $\alpha_1 + \alpha_2 + \dots +
\alpha_d$ for which the coefficient $c_{\alpha_1, \alpha_2, \dots, \alpha_d}$ is
non-zero. Following the terminology of~\citet{Klivans-Servedio-2008}, the
\emph{norm of a polynomial} $p$ is defined as
$$
\norm{p} = \sqrt{\sum_{\alpha_1, \alpha_2, \dots, \alpha_d} \left(c_{\alpha_1, \alpha_2, \dots, \alpha_d} \right)^2 } \; .
$$
It is easy see that this is indeed a norm, since we can interpret it as the
Euclidean norm of the vector of the coefficients of the polynomial.

\begin{theorem}[Polynomial approximation of intersection of halfspaces I]
\label{theorem:polynomial-approximation-1}
Let $v_1, v_2, \dots, v_m \in \R^d$ be vectors such that $\norm{v_1},
\norm{v_2}, \dots, \norm{v_m} \le 1$. Let $\gamma \in (0,1)$. There exists a
multivariate polynomial $p:\R^d \to \R$ such that
\begin{enumerate}
\item $p(x) \ge 1/2$ for all $\displaystyle x \in R^+ = \bigcap_{i=1}^m \left\{ x \in \B(\zero,1) ~:~  \ip{v_i}{x} \ge \gamma \right\}$,
\item $p(x) \le -1/2$ for all $\displaystyle x \in R^- = \bigcup_{i=1}^m \left\{ x \in \B(\zero,1) ~:~  \ip{v_i}{x} \le - \gamma \right\}$,
\item $\displaystyle \deg(p) = \left\lceil \log_2(2m) \right\rceil \cdot \left\lceil \sqrt{{1}/{\gamma}} \right\rceil$,
\item $\displaystyle \norm{p} \le \left[ 188 \left\lceil \log_2(2m) \right\rceil \cdot \left\lceil \sqrt{{1}/{\gamma}} \right\rceil \right]^{
  \frac{
    \left\lceil \log_2(2m) \right\rceil \cdot \left\lceil \sqrt{{1}/{\gamma}} \right\rceil
  }{2}
}$.
\end{enumerate}
\end{theorem}

\begin{theorem}[Polynomial approximation of intersection of halfspaces II]
\label{theorem:polynomial-approximation-2}
Let $v_1, v_2, \dots, v_m \in \R^d$ be vectors such that $\norm{v_1},
\norm{v_2}, \dots, \norm{v_m} \le 1$. Let $\gamma \in (0,1)$.
Define
$$
r = 2 \left\lceil \frac{1}{4} \log_2(4m + 1) \right\rceil + 1 \quad \text{and} \quad s = \left \lceil \log_2(1/\gamma) \right \rceil \; .
$$
Then, there exists a multivariate polynomial $p:\R^d \to \R$ such that
\begin{enumerate}
\item $\displaystyle p(x) \ge 1/2$
for all $\displaystyle x \in R^+ = \bigcap_{i=1}^m \left\{ x \in \B(\zero,1) ~:~ \ip{v_i}{x} \ge \gamma \right\}$,

\item $\displaystyle p(x) \le - 1/2$
for all $\displaystyle x \in R^- = \bigcup_{i=1}^m \left\{ x \in \B(\zero,1) ~:~ \ip{v_i}{x} \le - \gamma \right\}$,

\item $\deg(p) \le (2s+1) rm$,
\item $\norm{p} \le (4m-1) 2^m \cdot \left(2^s rm (4s+2) \right)^{(s+1/2)rm}$.
\end{enumerate}
\end{theorem}

The proofs of the theorems are in
Appendix~\ref{section:proof-of-polynomial-approximation}. The geometric
interpretation of the two regions $R^+$ and $R^-$ in the theorems is explained
in Figure~\ref{figure:pizza-slice}. Similar but weaker results were proved
by~\citet{Klivans-Servedio-2008}. Specifically, our bounds in parts 1, 2, 3, 4
of
Theorems~\ref{theorem:polynomial-approximation-1}~and~\ref{theorem:polynomial-approximation-2}
are independent of the dimension $d$.

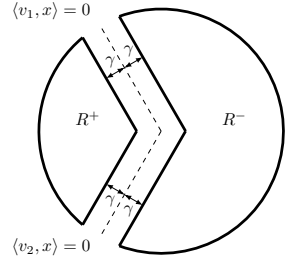
\begin{figure}
\begin{center}
 \scalebox{.65}{\begin{tikzpicture}[scale=0.5]

  \useasboundingbox (-5.5,-5.5) rectangle (5.5,5.5);
  % \draw[help lines] (-5.5,-5.5) rectangle (5.5,5.5);

  % Grid and coordinate axes
  %% \draw[help lines] (-5.2,-5.2) grid (5.2,5.2);
  %% \draw[thick, ->, >=latex] (-5.2,0) -- (5.2,0);
  %% \draw[thick, ->,  >=latex] (0,-5.2) -- (0,5.2);

  % Notable points
  \coordinate (Origin) at (0,0);

  \coordinate [label={[xshift=-10mm, yshift=0mm]$\ip{v_1}{x} = 0$}] (A) at (120:5);
  \coordinate [label={[xshift=-10mm, yshift=-5mm]$\ip{v_2}{x} = 0$}] (B) at (240:5);

  % Circle which contains all the examples.
  % \draw (Origin) circle (5cm);

  \draw[dashed] (Origin) -- (A);
  \draw[dashed] (Origin) -- (B);

  \coordinate (A1) at ($(120:0.5*sqrt{97} - 0.5) - (1,0)$);
  \coordinate (B1) at ($(240:0.5*sqrt{97} - 0.5) - (1,0)$);

  \coordinate (A2) at ($(120:0.5*sqrt{97} + 0.5) + (1,0)$);
  \coordinate (B2) at ($(240:0.5*sqrt{97} + 0.5) + (1,0)$);

  \draw[ultra thick] ($(Origin) - (1,0)$) -- (A1);
  \draw[ultra thick] ($(Origin) - (1,0)$) -- (B1);

  \draw[ultra thick] ($(Origin) + (1,0)$) -- (A2);
  \draw[ultra thick] ($(Origin) + (1,0)$) -- (B2);

  \pic [ultra thick, draw, angle radius=2.5cm] {angle=A1--Origin--B1};
  \pic [ultra thick, draw, angle radius=2.5cm] {angle=B2--Origin--A2};

  \coordinate (C) at (120:3);
  \coordinate (C1) at ($(120:3) + (0.75, 0.43301270189221932338)$);
  \coordinate (C2) at ($(120:3) - (0.75, 0.43301270189221932338)$);

  \draw [<->, >=latex, decoration={markings,mark=at position 1 with {\arrow[scale=0.5]{>}}},postaction={decorate}] (C1) -- (C);
  \draw [<->, >=latex, decoration={markings,mark=at position 1 with {\arrow[scale=0.5]{>}}},postaction={decorate}] (C2) -- (C);
  \draw ($0.5*(C1) + 0.5*(C)$) node[above, xshift=-1mm] {$\gamma$};
  \draw ($0.5*(C2) + 0.5*(C)$) node[above, xshift=-1mm] {$\gamma$};

  \coordinate (D) at (240:3);
  \coordinate (D1) at ($(240:3) + (-0.75, 0.43301270189221932338)$);
  \coordinate (D2) at ($(240:3) - (-0.75, 0.43301270189221932338)$);
  \draw [<->, >=latex, decoration={markings,mark=at position 1 with {\arrow[scale=0.5]{>}}},postaction={decorate}] (D1) -- (D);
  \draw [<->, >=latex, decoration={markings,mark=at position 1 with {\arrow[scale=0.5]{>}}},postaction={decorate}] (D2) -- (D);
  \draw ($0.5*(D1) + 0.5*(D)$) node[below, xshift=-1mm] {$\gamma$};
  \draw ($0.5*(D2) + 0.5*(D)$) node[below, xshift=-1mm] {$\gamma$};

  \coordinate [label={$R^+$}](R1) at ($(Origin) - (3,0)$);
  \coordinate [label={$R^-$}](R2) at ($(Origin) + (3,0)$);

\end{tikzpicture}}
\end{center}
\caption[]{The figure shows the two regions $R^+$ and $R^-$ used in parts 1 and
2 of
Theorems~\ref{theorem:polynomial-approximation-1}~and~\ref{theorem:polynomial-approximation-2}
for the case $m=d=2$ and a particular choice of vectors $v_1, v_2$ and margin
parameter $\gamma$. The separating hyperplanes $\ip{v_1}{x} = 0$ and
$\ip{v_2}{x} = 0$ are shown as dashed lines.}
\label{figure:pizza-slice}
\end{figure}

The following lemma establishes a correspondence between any multivariate
polynomial in $\R^d$ and an element in $\ell_2$, and gives an upper bound on its
norm. Its proof follows from simple algebra, which we defer to
Appendix~\ref{section:proof-norm-bound}.

\begin{lemma}[Norm bound]
\label{lemma:norm-bound}
Let $p:\R^d \to \R$ be a multivariate polynomial.
There exists $c \in \ell_2$ such that $p(x) = \ip{c}{\phi(x)}_{\ell_2}$
and $\norm{c}_{\ell_2} \le 2^{\deg(p)/2} \norm{p}$.
\end{lemma}

Using the lemma and the polynomial approximation theorems, we can prove that the
mapping $\phi$ maps any set of weakly linearly separable examples to a strongly
linearly separable set of examples. Due to space constraints, we defer the full
proof of Theorem~\ref{theorem:margin-transformation} to
Appendix~\ref{section:proof-of-theorem-margin-transformation}.

\section{Experiments}
\label{section:experiments}

In this section, we provide an empirical evaluation on our algorithms, verifying
their effectiveness on linearly separable datasets. We generated strongly and
weakly linearly separable datasets with $K=3$ classes in $\R^3$ i.i.d. from two
data distributions. Figures~\ref{figure:strongly-separable-dataset}
and~\ref{figure:weakly-separable-dataset} show visualizations of the two
datasets, along with detailed descriptions of the distributions.

\begin{figure}[h]
\centering
\begin{subfigure}[b]{0.23\textwidth}
\captionsetup{justification=centering}
\begin{center}
\includegraphics[width=\textwidth, trim={0, 0cm, 0, 0}, clip]{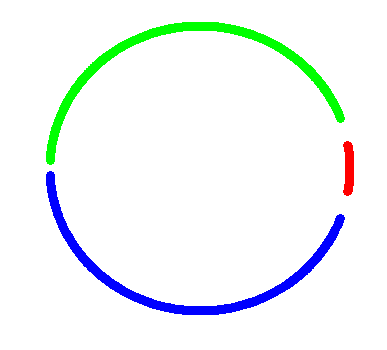}
\caption{Strongly separable case}
\label{figure:strongly-separable-dataset}
\end{center}
\end{subfigure}
\begin{subfigure}[b]{0.23\textwidth}
\captionsetup{justification=centering}
\centering
\includegraphics[width=\textwidth, trim={0, 0cm, 0, 0}, clip]{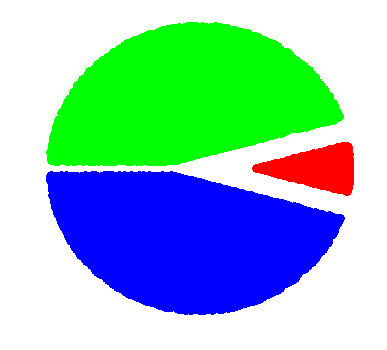}
\caption{Weakly separable case}
\label{figure:weakly-separable-dataset}
\end{subfigure}
\caption{Strongly and weakly linearly separable datasets in $\R^3$ with $K=3$
classes and $T=5\times 10^6$ examples. Here we show projections of the examples
onto their first two coordinates, which lie in the ball of radius $1/\sqrt{2}$
centered at the origin. The third coordinate is $1/\sqrt{2}$ for all examples.
Class 1 is depicted red. Classes 2 and 3 are depicted green and blue,
respectively. $80\%$ of the examples belong to class 1, $10\%$ belong to class 2
and $10\%$ belong to class 3. Class 1 lies in the angle interval $[-15^\circ,
15^\circ]$, while classes 2 and 3 lie in the angle intervals $[15^\circ,
180^\circ]$ and $[-180^\circ, -15^\circ]$ respectively. The examples are
strongly and weakly linearly separable with a margin of $\gamma=0.05$,
respectively. (Examples lying within margin $\gamma$ of the linear separators
were rejected during sampling.)}
\label{figure:strongly-and-weakly-separable-datasets}
\end{figure}

We implemented
Algorithm~\ref{algorithm:algorithm-for-strongly-linearly-separable-examples},
Algorithm~\ref{algorithm:kernelized} with rational
kernel~\eqref{equation:rational-kernel} and used implementation of
\textsc{Banditron} algorithm by \citet{Orabona09}. We evaluated these algorithms
on the two datasets. \textsc{Banditron} has an exploration rate parameter, for
which we tried values $0.02, 0.01, 0.005, 0.002, 0.001, 0.0005$. Since all three
algorithms are randomized, we run each algorithm $20$ times. The average
cumulative number of mistakes up to round $t$ as a function of $t$ are shown in
Figures~\ref{figure:number-of-mistakes-strongly-separable-dataset}
and~\ref{figure:number-of-mistakes-weakly-separable-dataset}.

We can see that there is a tradeoff in the setting of the exploration rate for
\textsc{Banditron}. With large exploration parameter, \textsc{Banditron} suffers
from over-exploration, whereas with small exploration parameter, its model
cannot be updated quickly enough. As expected,
Algorithm~\ref{algorithm:algorithm-for-strongly-linearly-separable-examples} has
a small number of mistakes in the strongly linearly separable setting, while
having a large number of mistakes in the weakly linearly separable setting, due
to the limited representation power of linear classifiers. In contrast,
Algorithm~\ref{algorithm:kernelized} with rational kernel has a small number of
mistakes in both settings, exhibiting strong adaptivity guarantees.
Appendix~\ref{section:supp-to-experiment} shows the decision boundaries that
each of the algorithms learns by the end of the last round.

\begin{minipage}{.48\textwidth}
\begin{figure}[H]
\centering
\includegraphics[width=0.98\textwidth]{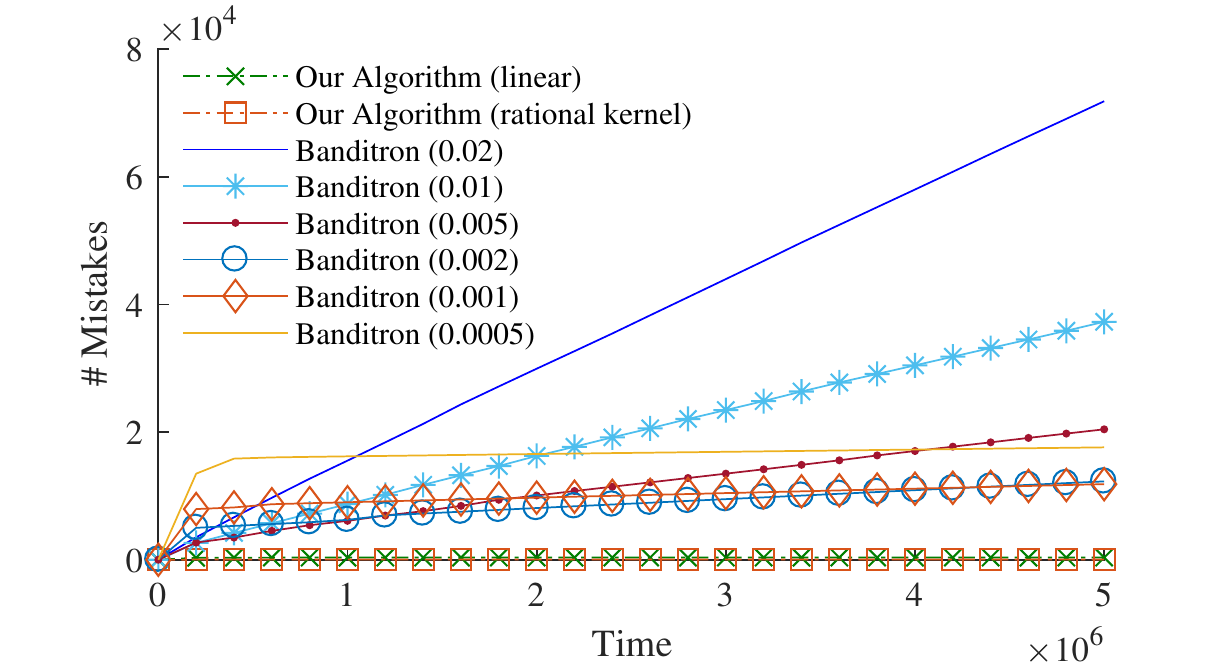}
\caption{The average cumulative number of mistakes versus the
number of rounds on the strongly linearly separable dataset in
\autoref{figure:strongly-separable-dataset}.}
\label{figure:number-of-mistakes-strongly-separable-dataset}
\end{figure}
\end{minipage}
\hfill
\begin{minipage}{.48\textwidth}
\begin{figure}[H]
\centering
\includegraphics[width=0.98\textwidth]{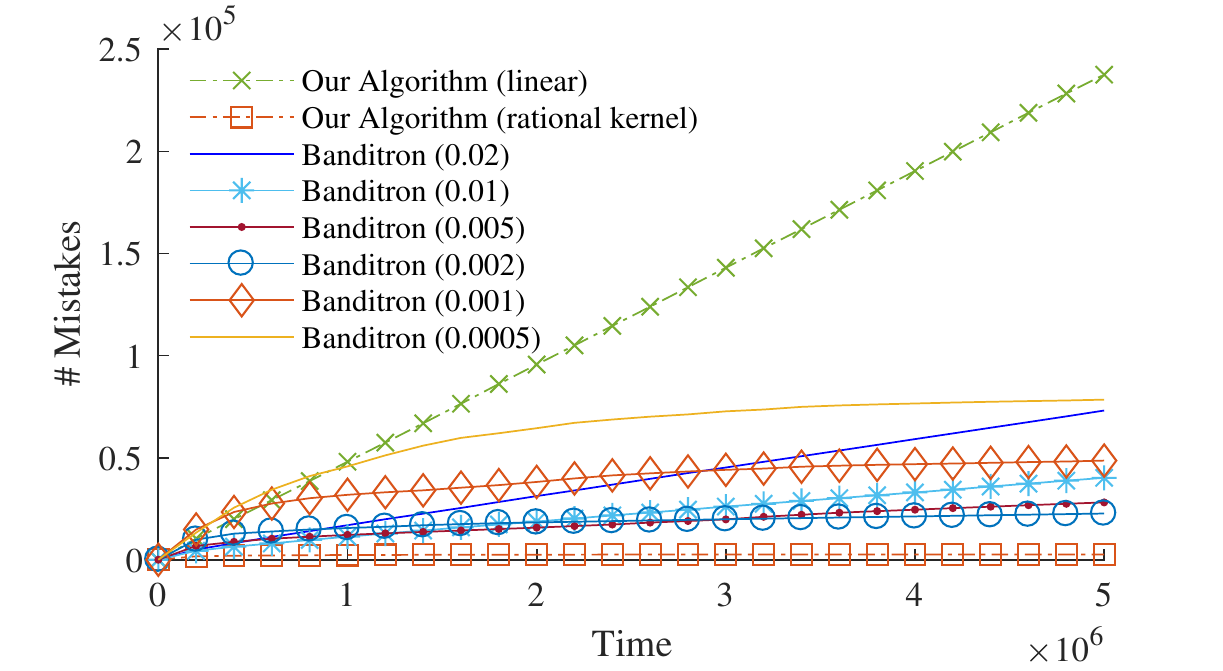}
\caption{The average cumulative number of mistakes versus the
number of rounds on the weakly linearly separable dataset in
\autoref{figure:weakly-separable-dataset}.}
\label{figure:number-of-mistakes-weakly-separable-dataset}
\end{figure}
\end{minipage}

\section*{Acknowledgments} We thank Francesco Orabona and Wen Sun
for helpful initial discussions, and thank Adam Klivans and Rocco \mbox{Servedio} for helpful
discussions on~\citep{Klivans-Servedio-2008} and
pointing out the reference~\citep{Klivans-Servedio-2004}. We also thank Dylan Foster, 
Akshay \mbox{Krishnamurthy}, and Haipeng Luo for providing a candidate solution to our problem. 
Finally, we thank Shang-En Huang and \mbox{Mengxiao} Zhang for helpful discussions on the hardness results. 

\bibliography{biblio}
\bibliographystyle{icml2019}

% For actual submission remove everything below.
% Use appendices only in supplementray material.
\onecolumn
\appendix
\clearpage

\section{Multiclass Perceptron}
\label{section:multiclass-perceptron-proofs}

\textsc{Multiclass Perceptron} is an algorithm for \textsc{Online Multiclass
Classification}. Both the protocol for the problem and the algorithm are stated
below. The algorithm assumes that the feature vectors come from an inner product
space $(V, \ip{\cdot}{\cdot})$.

Two results are folklore. The first result is
\autoref{theorem:multiclass-perceptron-mistake-upper-bound} which states that if
examples are linearly separable with margin $\gamma$ and examples have norm
at most $R$ then the algorithm makes at most $\lfloor 2 (R/\gamma)^2 \rfloor$
mistakes. The second result is
\autoref{theorem:online-multiclass-classification-mistake-lower-bound} which
states that under the same assumptions as in
\autoref{theorem:online-multiclass-classification-mistake-lower-bound}
\emph{any} deterministic algorithm for \textsc{Online Multiclass Classification}
must make at least $\lfloor (R/\gamma)^2 \rfloor$ mistakes in the worst case.

\begin{protocol}[h]
\caption{\textsc{Online Multiclass Classification}
\label{algorithm:mutliclass-classification}}
\textbf{Require:} Number of classes $K$, number of rounds $T$. \\
\textbf{Require:} Inner product space $(V,\ip{\cdot}{\cdot})$. \\
\For{$t=1,2,\dots,T$}{
Adversary chooses example $(x_t, y_t) \in V \times \{1,2,\dots,K\}$, where $x_t$ is revealed to the learner.\\
Predict class label $\widehat y_t \in \{1,2,\dots,K\}$.\\
Observe feedback $y_t$.
}
\end{protocol}

\begin{algorithm}[h]
\caption{\textsc{Multiclass Perceptron}
\label{algorithm:mutliclass-perceptron}}
\textbf{Require:} Number of classes $K$, number of rounds $T$. \\
\textbf{Require:} Inner product space $(V,\ip{\cdot}{\cdot})$. \\
Initialize $w_1^{(1)} = w_2^{(1)} = \dots = w_K^{(1)} = 0$ \\
\For{$t=1,2,\dots,T$}{
  Observe feature vector $x_t \in V$ \\
  Predict $\widehat y_t = \argmax_{i \in \{1,2,\dots,K\}} \ip{w_t^{(i)}}{x_t}$ \\
  Observe $y_t \in \{1,2,\dots,K\}$ \\
  \If{$\widehat y_t \neq y_t$}{
    Set $w_i^{(t+1)} = w_i^{(t)}$ \\ \qquad for all $i \in \{1,2,\dots,K\} \setminus \{y_t, \widehat y_t\}$ \\
    Update $w_{y_t}^{(t+1)} = w_{y_t}^{(t)} + x_t$ \\
    Update $w_{\widehat y_t}^{(t+1)} = w_{\widehat y_t}^{(t)} - x_t$ \\
  }
  \Else{
    Set $w_i^{(t+1)} = w_i^{(t)}$ for all $i \in \{1,2,\dots,K\}$ \\
  }
}
\end{algorithm}

\begin{theorem}[Mistake upper bound~\cite{Crammer-Singer-2003}]
\label{theorem:multiclass-perceptron-mistake-upper-bound}
Let $(V, \ip{\cdot}{\cdot})$ be an inner product space, let $K$ be a positive
integer, let $\gamma$ be a positive real number and let $R$ be a non-negative real
number. If $(x_1, y_1), (x_2, y_2), \dots, (x_T, y_T)$ is a sequence of labeled
examples in $V \times \{1,2,\dots,K\}$ that are weakly linearly separable with margin
$\gamma$ and $\norm{x_1}, \norm{x_2}, \dots, \norm{x_T} \le R$
then \textsc{Multiclass Perceptron} algorithm makes at most $\lfloor
2(R/\gamma)^2 \rfloor$ mistakes.
\end{theorem}

\begin{proof}
Let $M = \sum_{t=1}^T \indicator{\widehat y_t \neq y_t}$ be the number of
mistakes the algorithm makes. Since the $K$-tuple $(w_1^{(t)}, w_2^{(t)}, \dots,
w_K^{(t)})$ changes only if a mistake is made, we can upper bound $\sum_{i=1}^K
\norm{w_i^{(t)}}^2$ in terms of number of mistakes.
If a mistake happens in round $t$ then
\begin{align*}
\sum_{i=1}^K \norm{w_i^{(t+1)}}^2
& = \left(\sum_{i \in \{1,2,\dots,K\} \setminus \{y_t, \widehat y_t\} } \norm{w_i^{(t)}}^2 \right)
 + \norm{w_{y_t}^{(t)} + x_t}^2 + \norm{w_{\widehat y_t}^{(t)} - x_t}^2 \\
& = \left(\sum_{i \in \{1,2,\dots,K\} \setminus \{y_t, \widehat y_t\} } \norm{w_i^{(t)}}^2 \right) + \norm{w_{y_t}^{(t)}}^2 + \norm{w_{\widehat y_t}^{(t)}}^2 
 + 2 \norm{x_t}^2 + 2 \ip{w_{y_t}^{(t)} - w_{\widehat y_t}^{(t)}}{x_t} \\
& = \left(\sum_{i=1}^K \norm{w_i^{(t)}}^2 \right) + 2 \norm{x_t}^2 + 2 \ip{w_{y_t}^{(t)} - w_{\widehat y_t}^{(t)}}{x_t} \\
& \le \left(\sum_{i=1}^K \norm{w_i^{(t)}}^2 \right) + 2 \norm{x_t}^2 \\
& \le \left(\sum_{i=1}^K \norm{w_i^{(t)}}^2 \right) + 2 R^2 \; .
\end{align*}
So each time a mistake happens, $\sum_{i=1}^K \norm{w_i^{(t)}}^2$ increases by at most $2R^2$. Thus,
\begin{equation}
\sum_{i=1}^K \norm{w_i^{(T+1)}}^2 \le 2R^2 M \; .
\label{equation:perceptron-norm-ub}
\end{equation}
Let $w_1^*, w_2^*, \dots, w_K^* \in V$ be vectors satisfying
\eqref{equation:weak-linear-separability-1} and
\eqref{equation:weak-linear-separability-2}. We lower bound $\sum_{i=1}^K \ip{w_i^*}{w_i^{(t)}}$. This quantity changes
only when a mistakes happens. If mistake happens in round $t$, we have
\begingroup
\allowdisplaybreaks
\begin{align*}
\sum_{i=1}^K \ip{w_i^*}{w_i^{(t+1)}}
& = \left( \sum_{i \in \{1,2,\dots,K\} \setminus
\{y_t, \widehat y_t\}} \ip{w_i^*}{w_i^{(t)}} \right) \\
 & \quad + \ip{w_{y_t}^*}{w_{y_t}^{(t)} + x_t} + 
  \ip{w_{\widehat y_t}^*}{w_{\widehat y_t}^{(t)} - x_t} \\
& = \left( \sum_{i=1}^K \ip{w_i^*}{w_i^{(t)}} \right) + \ip{w_{y_t}^* - w_{\widehat y_t}^*}{x_t} \\
& \ge  \left( \sum_{i=1}^K \ip{w_i^*}{w_i^{(t)}} \right) + \gamma \; .
\end{align*}
\endgroup
Thus, after $M$ mistakes,
\[
\label{equation:perceptron-ip-lb}
\sum_{i=1}^K \ip{w_i^*}{w_i^{(T+1)}} \ge \gamma M \; .
\]
We upper bound the left hand side by using Cauchy-Schwartz inequality twice and
the condition \eqref{equation:weak-linear-separability-1} on $w_1^*, w_2^*, \dots,
w_K^*$. We have
\begin{align*}
\sum_{i=1}^K \ip{w_i^*}{w_i^{(T+1)}}
& \le \sum_{i=1}^K \norm{w_i^*} \cdot \norm{w_i^{(T+1)}} \\
& \le \sqrt{\sum_{i=1}^K \norm{w_i^*}^2} \sqrt{\sum_{i=1}^K \norm{w_i^{(T+1)}}^2} \\
& \le \sqrt{\sum_{i=1}^K \norm{w_i^{(T+1)}}^2} \; .
\end{align*}
Combining the above inequality with Equations~\eqref{equation:perceptron-norm-ub}
and~\eqref{equation:perceptron-ip-lb}, we get
$$
(\gamma M)^2 \le \sum_{i=1}^K \norm{w_i^{(T+1)}}^2 \le 2R^2 M \; .
$$
We conclude that $M \le 2(R/\gamma)^2$. Since $M$ is an integer, $M \le \lfloor 2(R/\gamma)^2 \rfloor$.
\end{proof}

\begin{theorem}[Mistake lower bound]
\label{theorem:online-multiclass-classification-mistake-lower-bound}
Let $K$ be a positive integer, let $\gamma$ be a positive real number and let
$R$ be a non-negative real number. For any (possibly randomized) algorithm
$\calA$ for the \textsc{Online Multiclass Classification} problem there exists
an inner product space $(V, \ip{\cdot}{\cdot})$, a non-negative integer $T$ and
a sequence of labeled examples $(x_1, y_1), (x_2, y_2), \dots, (x_T, y_T)$
examples in $V \times \{1,2,\dots,K\}$ that are weakly linearly separable with
margin $\gamma$, the norms satisfy $\norm{x_1}, \norm{x_2}, \dots, \norm{x_T}
\le R$ and the algorithm makes at least $\frac 1 2 \lfloor (R/\gamma)^2 \rfloor$
mistakes.
\end{theorem}

\begin{proof}
Let $T = \lfloor (R/\gamma)^2 \rfloor$, $V = \R^T$, and for all $t$ in
$\cbr{1,\ldots,T}$, define instance $x_t = R e_t$ where $e_t$ is $t$-th element
of the standard orthonormal basis of $\R^T$.
Let labels $y_1, \ldots, y_T$ be chosen i.i.d uniformly at random from
$\cbr{1,2,\ldots,K}$ and independently of any randomness used by the algorithm
$\calA$.

%let $x_1, x_2, \dots, x_T$ be the
%orthogonal vectors such that $\norm{x_t} = R$ for all $t=1,2,\dots,T$. For
%example, we can take

%Since the algorithm is deterministic, we can construct the sequence of labels
%$y_1, y_2, \dots, y_T$ adaptively based on the predictions $\widehat y_1,
%\widehat y_2, \dots, \widehat y_T$ of the algorithm. We define $y_t$ to be any
%element of $\{1,2,\dots,K\}$ not equal to $\widehat y_t$. This way the algorithm
%makes a mistake in every round $t=1,2,\dots,T$.

We first show that the set of examples $(x_1,y_1)$, $\ldots$,
$(x_T, y_T)$ we have constructed is weakly linearly separable
with margin $\gamma$. To prove that, we demonstrate vectors $w_1, w_2, \dots, w_K$
satisfying conditions \eqref{equation:weak-linear-separability-1} and
\eqref{equation:weak-linear-separability-2}. We define
$$
w_i = \frac{\gamma}{R} \sum_{\substack{t : 1 \le t \le T \\ y_t = i}} e_t \qquad \text{for $i=1,2,\dots,K$.}
$$
Let $a_i = |\{ t ~:~ 1 \le t \le T, \ y_t = i \}|$ be the number of occurrences of label $i$.
It is easy to see that
$$
\norm{w_i}^2 = \frac{\gamma^2}{R^2} \sum_{\substack{t : 1 \le t \le T \\ y_t = i}} \norm{e_t}^2 = \frac{a_i \gamma^2}{R^2} \qquad \text{for $i=1,2,\dots,K$.}
$$
Since $\sum_{i=1}^K a_i = T$,
$\sum_{i=1}^K \norm{w_i}^2 = T \cdot \frac{\gamma^2}{R^2} \leq 1$, i.e.
the condition
\eqref{equation:weak-linear-separability-1} holds. To verify condition
\eqref{equation:weak-linear-separability-2} consider any labeled example $(x_t,
y_t)$. Then, for any $i$ in $\cbr{1,\ldots,K}$, by the definition of $w_i$, we have
\begin{align*}
\ip{w_i}{x_t}
& = \frac{\gamma}{R} \sum_{\substack{s : 1 \le s \le T \\ y_s = i}} \ip{e_s}{R e_t} \\
& = \gamma \cdot \sum_{\substack{s : 1 \le s \le T \\ y_s = i}} \indicator{s = t} \\
& = \gamma \cdot \indicator{y_t = i}
\; .
\end{align*}
Therefore, if $i = y_t$, $\ip{w_i}{x_t} = \gamma$;
otherwise $i \neq y_t$, in which case $\ip{w_i}{x_t} = 0$.
Hence, condition \eqref{equation:weak-linear-separability-2} holds.

We now give a lower bound on the number of mistakes $\calA$ makes.
As $y_t$ is chosen uniformly from $\{1,2,\dots,K\}$, independently from
$\calA$'s randomization and the first $t-1$ examples,
$$
\Exp[ \indicator{\widehat y_t \neq y_t} ] \ge 1 - \frac{1}{K} \ge \frac{1}{2} \; .
$$
Summing over all $t$ in $\cbr{1,\ldots,T}$, we conclude that
$$
\Exp \sbr{ \sum_{t=1}^T \indicator{\widehat y_t \neq y_t} } \geq \frac T 2 = \frac 1 2 \lfloor (R/\gamma)^2 \rfloor,
$$
which completes the proof.
\end{proof}

%& = \frac{R}{\sqrt{T}} \\
%& \ge \frac{R}{R/\gamma} \\
%& = \gamma \; .
%and similarly, for any $i \in \{1,2,\dots,K\} \setminus \{y_t\}$ we have
%$\ip{w_i}{x_t} = 0$.

\section{Proofs of Theorems~\ref{theorem:strongly-separable-examples-mistake-upper-bound} and \ref{theorem:strongly-separable-examples-mistake-lower-bound}}
\label{section:proofs-for-stringly-separable-examples}

\begin{proof}[Proof of \autoref{theorem:strongly-separable-examples-mistake-upper-bound}]
Let $M = \sum_{t=1}^T z_t$ be the number of mistakes
Algorithm~\ref{algorithm:algorithm-for-strongly-linearly-separable-examples}
makes. Let $A = \sum_{t=1}^T \indicator{S_t \neq \emptyset} z_t$ be the number of
mistakes in the rounds when $S_t \neq \emptyset$, i.e. the number of rounds
line~\ref{line:neg-update} is executed. In addition, let $B = \sum_{t=1}^T
\indicator{S_t = \emptyset} z_t$ be the number of mistakes in the rounds when $S_t =
\emptyset$. It can be easily seen that $M = A + B$.

Let $C = \sum_{t=1}^T \indicator{S_t = \emptyset}(1 - z_t)$ be the number of rounds
line~\ref{line:pos-update} gets executed. Let $U = \sum_{t=1}^T (\indicator{S_t \neq
\emptyset} z_t + \indicator{S_t = \emptyset}(1 - z_t))$ be the number of rounds
line~\ref{line:pos-update} or~\ref{line:neg-update} gets executed. In other
words, $U$ is the number of times the $K$-tuple of vectors $(w_1^{(t)},
w_2^{(t)}, \dots, w_K^{(t)})$ gets updated. It can be easily seen that $U = A +
C$.

The key observation is that $\Exp[B] = (K-1) \Exp[C]$.
To see this, note that if $S_t = \emptyset$, there is $1/K$ probability that the algorithm
guesses the correct label ($z_t = 0$) and with probability $(K-1)/K$ algorithm's guess is
incorrect ($z_t = 1$). Therefore,
\[ \Exp[ z_t |S_t = \emptyset] = \frac{K-1}{K}, \]
\[ \Exp[B] = \frac{K-1}{K} \Exp \sbr{ \sum_{t=1}^T \indicator{S_t = \emptyset} }, \]
\[ \Exp[C] = \frac{1}{K} \Exp \sbr{ \sum_{t=1}^T \indicator{S_t = \emptyset} }. \]

Putting all the information together, we get that
\begin{align}
\Exp[M]
& = \Exp[A] + \Exp[B] \nonumber \\
& = \Exp[A] + (K-1) \Exp[C] \nonumber \\
& \le (K-1) \Exp[A + C] \nonumber\\
& = (K-1) \Exp[U]  \; .
\label{eqn:mistake-update}
\end{align}

To finish the proof, we need to upper bound the number of updates $U$. We claim
that $U \le \lfloor 4(R/\gamma)^2 \rfloor$ with probability 1.
The proof of this upper bound is
similar to the proof of the mistake bound for \textsc{Multiclass Perceptron}
algorithm. Let $w_1^*, w_2^*, \dots, w_K^* \in V$ be vectors that satisfy
\eqref{equation:strong-linear-separability-1},
\eqref{equation:strong-linear-separability-2} and
\eqref{equation:strong-linear-separability-3}.
The $K$-tuple $(w_1^{(t)}, w_2^{(t)}, \dots, w_K^{(t)})$
changes only if there is an update in round $t$.
We investigate how $\sum_{i=1}^K \norm{w_i^{(t)}}^2$ and
$\sum_{i=1}^K \ip{w_i^*}{w_i^{(t)}}$ change. If there is an update in round $t$,
by lines~\ref{line:pos-update} and~\ref{line:neg-update}, we always have
$ w_{\widehat y_t}^{(t+1)} = w_{\widehat y_t}^{(t)} + (-1)^{z_t} x_t $,
and for all $i \neq \widehat y_t$, $w_{i}^{(t+1)} = w_{i}^{(t)}$.
Therefore,
\begingroup
\allowdisplaybreaks
\begin{align*}
\sum_{i=1}^K \norm{w_i^{(t+1)}}^2
& = \left( \sum_{i \in \{1,2,\dots,K\} \setminus \{\widehat y_t\}} \norm{w_i^{(t)}}^2 \right) + \norm{w_{\widehat y_t}^{(t+1)}}^2 \\
& = \left( \sum_{i \in \{1,2,\dots,K\} \setminus \{\widehat y_t\}} \norm{w_i^{(t)}}^2 \right) + \norm{w_{\widehat y_t}^{(t)} + (-1)^{z_t} x_t}^2 \\
& = \left( \sum_{i=1}^K \norm{w_i^{(t)}}^2 \right) + \norm{x_t}^2 + \underbrace{(-1)^{z_t} 2 \ip{w_{\widehat y_t}^{(t)}}{x_t}}_{\le 0} \\
& \le \left( \sum_{i=1}^K \norm{w_i^{(t)}}^2 \right) + \norm{x_t}^2 \\
& \le \left( \sum_{i=1}^K \norm{w_i^{(t)}}^2 \right) + R^2 \; .
\end{align*}
\endgroup
The inequality that $(-1)^{z_t} 2 \ip{w_{\widehat y_t}^{(t)}}{x_t} \leq 0$ is from a case analysis: if line~\ref{line:pos-update} is executed, then $z_t = 0$ and $\ip{w_{\widehat y_t}^{(t)}}{x_t} < 0$;
 otherwise line~\ref{line:neg-update} is executed, in which case $z_t = 1$ and $\ip{w_{\widehat y_t}^{(t)}}{x_t} \ge 0$.

Hence, after $U$ updates,
\begin{equation}
\sum_{i=1}^K \norm{w_i^{(T+1)}}^2 \le R^2 U \; .
\label{eqn:norm-ub}
\end{equation}
Similarly, if there is an update in round $t$, we have
\begingroup
\allowdisplaybreaks
\begin{align*}
\sum_{i=1}^K \ip{w_i^*}{w_i^{(t)}}
& = \left( \sum_{i \in \{1,2,\dots,K\} \setminus \{\widehat y_t\}} \ip{w_i^*}{w_i^{(t)}} \right) + \ip{w_{\widehat y_t}^*}{w_{\widehat y_t}^{(t+1)}} \\
& = \left( \sum_{i \in \{1,2,\dots,K\} \setminus \{\widehat y_t\}} \ip{w_i^*}{w_i^{(t)}} \right) + \ip{w_{\widehat y_t}^*}{w_{\widehat y_t}^{(t)} + (-1)^{z_t} x_t} \\
& = \left( \sum_{i=1}^K \ip{w_i^*}{w_i^{(t)}} \right) + (-1)^{z_t} \ip{w_{\widehat y_t}^*}{x_t} \\
& \ge \left( \sum_{i=1}^K \ip{w_i^*}{w_i^{(t)}} \right) + \frac \gamma 2,
\end{align*}
\endgroup
where the last inequality follows from a case analysis on $z_t$ and
Definition~\ref{definition:linear-separability}: if $z_t = 0$, then
$\widehat y_t = y_t$, by Equation~\eqref{equation:strong-linear-separability-2},
we have that $\ip{w_{\widehat y_t}^*}{x_t} \geq \frac \gamma 2$; if $z_t = 1$,
then $\widehat y_t \neq y_t$, by
Equation~\eqref{equation:strong-linear-separability-3}, we have that
$\ip{w_{\widehat y_t}^*}{x_t} \le -\frac \gamma 2$.

Thus, after $U$ updates,
\begin{equation}
\sum_{i=1}^K \ip{w_i^*}{w_i^{(T+1)}} \ge \frac {\gamma U} 2 \; .
\label{eqn:norm-lb}
\end{equation}
Applying Cauchy-Schwartz's inequality twice, and using assumption
\eqref{equation:strong-linear-separability-1}, we get that
\begin{align*}
\sum_{i=1}^K \ip{w_i^*}{w_i^{(T+1)}}
& \le \sum_{i=1}^K \norm{w_i^*} \cdot \norm{w_i^{(T+1)}} \\
& \le \sqrt{\sum_{i=1}^K \norm{w_i^*}^2} \sqrt{\sum_{i=1}^K \norm{w_i^{(T+1)}}^2} \\
& \le \sqrt{\sum_{i=1}^K \norm{w_i^{(T+1)}}^2} \; .
\end{align*}
Combining the above inequality with Equations~\eqref{eqn:norm-ub} and~\eqref{eqn:norm-lb}, we get
$$
\left(\frac{\gamma U}{2} \right)^2 \le \sum_{i=1}^K \norm{w_i^{(T+1)}}^2 \le R^2 U \; .
$$
We conclude that $U \le 4(R/\gamma)^2$. Since $U$ is an integer, $U \le \lfloor 4(R/\gamma)^2 \rfloor$.

Applying Equation~\eqref{eqn:mistake-update}, we get
$$
\Exp[M] \leq (K-1) \Exp[U] \leq (K-1) \lfloor 4(R/\gamma)^2 \rfloor \; . \qedhere
$$
\end{proof}

\begin{proof}[Proof of \autoref{theorem:strongly-separable-examples-mistake-lower-bound}]
%We use probabilistic method.
Let $M = \left\lfloor \frac{1}{4} (R/\gamma)^2
\right\rfloor$. Let $V = \R^{M+1}$ equipped with the standard inner product.
Let $e_1, e_2, \dots, e_{M+1}$ be the standard orthonormal basis of $V$. We
define vectors $v_1, v_2, \dots, v_M \in V$ where $v_j = \frac{R}{\sqrt{2}}(e_j
+ e_{M+1})$ for $j=1,2,\dots,M$. Let $\ell_1, \ell_2, \dots, \ell_M$ be chosen
i.i.d. uniformly at random from $\{1,2,\dots,K\}$ and independently of any
randomness used the by algorithm $\calA$. Let $T = M (K - 1)$. We define examples $(x_1,
y_1), (x_2, y_2), \dots, (x_T, y_T)$ as follows. For any $j=1,2,\dots,M$ and any
$h=1,2,\dots,K-1$,
$$
(x_{(j-1)(K-1) + h}, y_{(j-1)(K-1) + h}) = (v_j, \ell_j)
$$

The norm of each example is exactly $R$. The examples are strongly linearly separable
with margin $\gamma$. To see that, consider $w_1^*, w_2^*, \dots, w_K^* \in V$
defined by
$$
w_i^* = \sqrt{2} \frac{\gamma}{R} \left( \sum_{j ~:~ \ell_j = i} e_j \right) - \frac{\sqrt{2}}2 \frac{\gamma}{R} e_{M+1}
$$
for $i=1,2,\dots,K$.

For $i \in \{1,2,\dots,K\}$ and $j \in
\{1,2,\dots,M\}$, consider the inner product of $w_i^*$ and $v_j$.
If $i = \ell_j$, $\ip{w_i^*}{v_j} = \gamma - \frac \gamma 2 = \frac \gamma 2$;
otherwise $i \neq \ell_j$, in which case
$\ip{w_i^*}{v_j} = 0 - \frac \gamma 2 = - \frac \gamma 2$.
This means that $w_1^*, w_2^*, \dots, w_K^*$ satisfy
conditions
\eqref{equation:strong-linear-separability-2} and
\eqref{equation:strong-linear-separability-3}. Condition \eqref{equation:strong-linear-separability-1}
is satisfied since
\begin{align*}
\sum_{i=1}^K \norm{w_i^*}^2
& = 2 \frac{\gamma^2}{R^2} \sum_{j=1}^M \norm{e_j}^2 +  \frac{\gamma^2}{2 R^2} K \norm{e_{M+1}}^2 \\
& = 2 \frac{\gamma^2}{R^2} M + \frac{\gamma^2}{2 R^2} K
\le \frac{1}{2} + \frac{1}{2}
= 1 \; .
\end{align*}

It remains to lower bound the expected number of mistakes of $\calA$. For
any $j \in \{1,2,\dots,M\}$, consider the expected number of mistakes the
algorithm makes in rounds $(K-1)(j-1) + 1, (K-1)(j-1) + 2, \dots, (K-1)j$.

Define a filtration of $\sigma$-algebras $\cbr{\calB_j}_{j=0}^M$, where $\calB_j
= \sigma((x_1, y_1, \hat{y}_1), \ldots, (x_{(K-1)j}, y_{(K-1)j},
\hat{y}_{(K-1)j}))$ for every $j$ in $\{1,2,\dots,M\}$. By Claim 2
of~\citet{Daniely-Helbertal-2013}, as $\ell_j$ is chosen uniformly from
$\cbr{1,\dots,K}$ and independent of $\calB_{j-1}$ and $\calA$'s randomness,
$$
\Exp \sbr{ \sum_{t=(K-1)(j-1) + 1}^{(K-1)j} z_t ~\Bigg|~ \calB_{j-1} } \ge \frac{K-1}{2} \; .
$$
This implies that
$$
\Exp \sbr{ \sum_{t=(K-1)(j-1) + 1}^{(K-1)j} z_t } \ge \frac{K-1}{2} \; .
$$
Summing over all $j$ in $\{1,2,\dots,M\}$,
$$
\Exp \sbr{ \sum_{t=1}^{(K-1)M} z_t} \geq \frac{K-1}2 \cdot M = \frac{K-1}2 \left\lfloor \frac 1 4 (R/\gamma)^2 \right\rfloor \; .
$$

Thus there exists a particular sequence of examples for which
the algorithm makes at least $\frac{K-1}2 \left\lfloor \frac 1 4 (R/\gamma)^2
\right\rfloor$ mistakes in expectation over its internal randomization.
\end{proof}

\section{Proof of Lemma~\ref{lemma:norm-bound}}
\label{section:proof-norm-bound}

\begin{proof}
Note that the polynomial $p$ can be written as
$p(x) = \sum_{\alpha_1, \alpha_2, \dots, \alpha_d} c'_{\alpha_1, \alpha_2, \dots, \alpha_d} x_1^{\alpha_1} x_2^{\alpha_2} \dots x_d^{\alpha_d}$.
We define $c \in \ell_2$ using the multi-index notation as
$$
c_{\alpha_1, \alpha_2, \dots, \alpha_d}
= \frac{c'_{\alpha_1, \alpha_2, \dots, \alpha_d} 2^{(\alpha_1 + \alpha_2 + \dots + \alpha_d)/2}}{\sqrt{\binom{\alpha_1 + \alpha_2 + \dots + \alpha_d}{\alpha_1, \alpha_2, \dots, \alpha_d}}}
$$
for all tuples $(\alpha_1, \alpha_2, \dots, \alpha_d)$ such that $\alpha_1 + \alpha_2 + \dots + \alpha_d \le \deg(p)$.
Otherwise, we define $c_{\alpha_1, \alpha_2, \dots, \alpha_d} = 0$. By the definition
of $\phi$, $\ip{c}{\phi(x)}_{\ell_2} = p(x)$.

Whether $\alpha_1 + \ldots + \alpha_d \leq \deg(p)$, we always have:
\begin{align*}
|c_{\alpha_1, \alpha_2, \dots, \alpha_d}|
 \le 2^{(\alpha_1 + \alpha_2 + \dots + \alpha_d)/2} |c'_{\alpha_1, \alpha_2, \dots, \alpha_d}|
 \le 2^{\deg(p)/2} |c'_{\alpha_1, \alpha_2, \dots, \alpha_d}| \; .
\end{align*}
Therefore,
\begin{align*}
\norm{c}_{\ell_2}
 \le 2^{\deg(p)/2} \sqrt{\sum_{\alpha_1, \alpha_2, \dots, \alpha_d} (c'_{\alpha_1, \alpha_2, \dots, \alpha_d})^2} 
 = 2^{\deg(p)/2} \norm{p} \; . \qquad \qedhere
\end{align*}
\end{proof}

\section{Proofs of Theorems~\ref{theorem:polynomial-approximation-1}~and~\ref{theorem:polynomial-approximation-2}}
\label{section:proof-of-polynomial-approximation}

In this section, we follow the construction of~\citet{Klivans-Servedio-2008}
(which in turn uses the constructions of~\citet{Beigel-Reingold-Spielman-1995})
to establish two polynomials of low norm, such that it takes large positive values
in
\[ \bigcap_{i=1}^m \left\{ x \in \R^d ~:~ \norm{x} \le 1, \ \ip{v_i}{x} \ge \gamma \right\} \]
and takes large negative values in
\[ \bigcup_{i=1}^m \left\{ x \in \R^d ~:~ \norm{x} \le 1, \ \ip{v_i}{x} \le -\gamma \right\} . \]
We improve the norm bound analysis of~\citet{Klivans-Servedio-2008} in two aspects:
\begin{enumerate}
  \item Our upper bounds on the norm of the polynomials do not have any dependency on the
  dimensionality $d$.
  \item We remove the requirement that the fractional part of input $x$ must be above some threshold in
  Theorem~\ref{theorem:polynomial-approximation-2}.
\end{enumerate}
A lot of the proof details are similar to those of~\citet{Klivans-Servedio-2008}; nevertheless,
we provide a self-contained full proof here.

For the proofs of the theorems we need several auxiliary results.

\begin{lemma}[Simple inequality]
\label{lemma:simple-inequality}
For any real numbers $b_1, b_2, \dots, b_n$,
$$
\left( \sum_{i=1}^n b_i \right)^2 \le n \sum_{i=1}^n b_i^2 \; .
$$
\end{lemma}

\begin{proof}
The lemma follows from Cauchy-Schwartz inequality applied to
vectors $(b_1, b_2, \dots, b_n)$ and $(1,1,\dots,1)$.
\end{proof}

\begin{lemma}[Bound on binomial coefficients]
\label{lemma:binomial-bound}
For any integers $n,k$ such that $n \ge k \ge 0$,
$$
\binom{n}{k} \le (n - k + 1)^k \; .
$$
\end{lemma}

\begin{proof}
If $k = 0$, the inequality trivially holds. For the rest of the proof we can
assume $k \ge 1$. We write the binomial coefficient as
\begin{align*}
\binom{n}{k}
& = \frac{n(n-1)\cdots(n-k+1)}{k(k-1) \cdots 1} \\
& = \frac{n}{k} \cdot \frac{n-1}{k - 1} \cdots \frac{n-k+1}{1} \; .
\end{align*}
We claim that
$$
\frac{n}{k} \le \frac{n-1}{k - 1} \le \cdots \le \frac{n-k+1}{1}
$$
from which the lemma follows by upper bounding all the fractions by $n-k+1$.
It remains to prove that for any $j=0,1,\dots,k-1$,
$$
\frac{n - j + 1}{k - j + 1} \le \frac{n - j}{k - j} \; .
$$
Multiplying by the (positive) denominators, we get an equivalent inequality
$$
(n - j + 1)(k - j) \le (n - j)(k - j + 1) \; .
$$
We multiply out the terms and get
$$
nk - kj + k - nj + j^2 - j \le nk - nj + n - kj + j^2 - j \; .
$$
We cancel common terms and get an equivalent inequality $k \ge n$, which
holds by the assumption.
\end{proof}

\begin{lemma}[Properties of the norm of polynomials]
\label{lemma:properties-of-norm-of-polynomials}
\hspace{1cm} % Dummy space
\begin{enumerate}
\item Let $p_1, p_2, \dots, p_n$ be multivariate polynomials and let $p(x) =
\prod_{j=1}^n p_j(x)$ be their product.  Then, $\norm{p}^2 \le n^{\sum_{j=1}^n
\deg(p_j)} \prod_{j=1}^n \norm{p_j}^2$.

\item Let $q$ be a multivariate polynomial of degree at most $s$ and let $p(x) =
(q(x))^n$. Then, $\norm{p}^2 \le n^{ns} \norm{q}^{2n}$.

\item Let be $p_1, p_2, \dots, p_n$ be multivariate polynomials. Then,
$\norm{\sum_{j=1}^n p_j} \le \sum_{j=1}^n \norm{p_j}$.
Consequently,
$\norm{\sum_{j=1}^n p_j}^2 \le n \sum_{j=1}^n \norm{p_j}^2$.
\end{enumerate}
\end{lemma}

\begin{proof}
Using multi-index notation we can write any multivariate polynomial $p$ as
$$
p(x) = \sum_A c_A x^A
$$
where $A = (\alpha_1, \alpha_2, \dots, \alpha_d)$ is a multi-index (i.e. a $d$-tuple of
non-negative integers), $x^A = x_1^{\alpha_1} x_2^{\alpha_2} \dots x_d^{\alpha_d}$ is a
monomial and $c_A = c_{\alpha_1, \alpha_2, \dots, \alpha_d}$ is the corresponding real
coefficient. The sum is over a finite subset of $d$-tuples of non-negative
integers. Using this notation, the norm of a polynomial $p$ can be written as
$$
\norm{p} = \sqrt{\sum_A (c_A)^2} \; .
$$
For a multi-index $A = (\alpha_1, \alpha_2, \dots, \alpha_d)$ we define its
$1$-norm as $\norm{A}_1 = \alpha_1 + \alpha_2 + \dots + \alpha_d$.

To prove the part 1, we express $p_j$ as
$$
p_j(x) = \sum_{A_j} c^{(j)}_{A_j} x^{A_j} \; .
$$
Since $p(x) = \prod_{i=1}^n p_j(x)$, the coefficients of its expansion $p(x) =
\sum_A c_A x^A$ are
$$
c_A = \sum_{\substack{(A_1, A_2, \dots, A_n) \\ A_1 + A_2 + \dots + A_n = A}} c^{(1)}_{A_1} c^{(2)}_{A_2} \cdots c^{(n)}_{A_n} \; .
$$
Therefore,
\begin{align*}
\norm{p}^2
& = \sum_{A} (c_A)^2 \\
& = \sum_{A} \left( \sum_{\substack{(A_1, A_2, \dots, A_n) \\ A_1 + A_2 + \dots + A_n = A}} c^{(1)}_{A_1} c^{(2)}_{A_2} \cdots c^{(n)}_{A_n} \right)^2 \\
& = \sum_{A} \left( \sum_{\substack{(A_1, A_2, \dots, A_n) \\ A_1 + A_2 + \dots + A_n = A}} \prod_{j=1}^n c^{(j)}_{A_j} \right)^2
\end{align*}
and
\begin{align*}
\prod_{i=1}^n \norm{p_i}^2
& = \prod_{i=1}^n \left( \sum_{A_i} (c^{(i)}_{A_i})^2 \right) \\
& = \sum_{(A_1, A_2, \dots, A_n)} \prod_{j=1}^n (c^{(j)}_{A_j})^2 \\
& = \sum_{(A_1, A_2, \dots, A_n)} \left( \prod_{j=1}^n c^{(j)}_{A_j} \right)^2 \\
& = \sum_A \sum_{\substack{(A_1, A_2, \dots, A_n) \\ A_1 + A_2 + \dots + A_n = A}} \left( \prod_{j=1}^n c^{(j)}_{A_j} \right)^2
\end{align*}
where in both cases the outer sum is over multi-indices $A$ such that $\norm{A}_1 \le \deg(p)$.
\autoref{lemma:simple-inequality} implies that for any multi-index $A$,
\[
\left( \sum_{\substack{(A_1, A_2, \dots, A_n) \\ A_1 + A_2 + \dots + A_n = A}} \prod_{j=1}^n c^{(j)}_{A_j} \right)^2
\le M_A \sum_{\substack{(A_1, A_2, \dots, A_n) \\ A_1 + A_2 + \dots + A_n = A}} \left( \prod_{j=1}^n c^{(j)}_{A_j} \right)^2 \; .
\]
where $M_A$ is the number of $n$-tuples $(A_1, A_2, \dots, A_n)$ such that $A_1 +
A_2 + \dots + A_n = A$.

To finish the proof, it is sufficient to prove that $M_A \le n^{\deg(p)}$ for
any $A$ such that $\norm{A}_1 \le \deg(p)$. To prove this inequality, consider a
multi-index $A = (\alpha_1, \alpha_2, \dots, \alpha_d)$ and consider its $i$-th coordinate
$\alpha_i$. In order for $A_1 + A_2 + \dots + A_n = A$ to hold, the $i$-th
coordinates of $A_1, A_2, \dots, A_n$ need to sum to $\alpha_i$. There are exactly
$\binom{\alpha_i + n - 1}{\alpha_i}$ possibilities for the choice of $i$-th
coordinates of $A_1, A_2, \dots, A_n$. The total number of choices is thus
$$
M_A = \prod_{i=1}^d \binom{\alpha_i + n - 1}{\alpha_i} \; .
$$
Using \autoref{lemma:binomial-bound}, we upper bound it as
$$
M_A \le \prod_{i=1}^d n^{\alpha_i} = n^{\norm{A}_1} \le n^{\deg(p)} \; .
$$

Part 2 follows from the part 1 by setting $p_1 = p_2 = \dots p_n = q$.

The first inequality of part 3 follows from triangle inequality in Euclidean spaces, by viewing the polynomials
$p = \sum_{A} c_A x^A$ as multidimensional vectors $(c_A)$, and $\| p \| = \| (c_A) \|$.

For the second inequality, by~\autoref{lemma:simple-inequality}, we have
\[
\norm{\sum_{j=1}^n p_j}^2 = \left( \norm{\sum_{j=1}^n p_j} \right)^2 \le \left(\sum_{j=1}^n \norm{p_j} \right)^2 \le n \sum_{j=1}^n \norm{p_j}^2 \; .
\]
\end{proof}

\subsection{Proof of \autoref{theorem:polynomial-approximation-1}}
\label{section:proof-of-polynomial-approximation-1}

To construct the polynomial $p$ we use Chebyshev polynomials of the first kind.
Chebyshev polynomials of the fist kind form an infinite sequence of polynomials
$T_0(z), T_1(z), T_2(z), \dots$ of single real variable $z$. They are defined
by the recurrence
\begin{align*}
T_0(z) & = 1  \; , \\
T_1(z) & = z  \; ,\\
T_{n+1}(z) & = 2zT_n(z) - T_{n-1}(z), \quad \text{for $n \ge 1$.}
\end{align*}
Chebyshev polynomials have a lot of interesting properties.
We will need properties listed in
\autoref{proposition:properties-of-chebyshev-polynomials} below.
Interested reader can learn more about Chebyshev polynomials
from the book by \citet{Mason-Handscomb-2002}.

\begin{proposition}[Properties of Chebyshev polynomials]
\label{proposition:properties-of-chebyshev-polynomials}
Chebyshev polynomials satisfy
\begin{enumerate}
\item $\deg(T_n) = n$ for all $n \ge 0$.
\item If $n \ge 1$, the leading coefficient of $T_n(z)$ is $2^{n-1}$.
\item $T_n(\cos(\theta)) = \cos(n \theta)$ for all $\theta \in \R$ and all $n \ge 0$.
\item $T_n(\cosh(\theta)) = \cosh(n \theta)$ for all $\theta \in \R$ and all $n \ge 0$.
\item $|T_n(z)| \le 1$ for all $z \in [-1,1]$ and all $n \ge 0$.
\item $T_n(z) \ge 1 + n^2(z - 1)$ for all $z \ge 1$ and all $n \ge 0$.
\item $\norm{T_n} \le (1+\sqrt{2})^n$ for all $n \ge 0$
\end{enumerate}
\end{proposition}

\begin{proof}[Proof of \autoref{proposition:properties-of-chebyshev-polynomials}]
The first two properties can be easily proven by induction on $n$ using the recurrence.

We prove the third property by induction on $n$. Indeed, by definition
$$
T_0(\cos(\theta)) = 1 = \cos(0 \theta) \quad \text{and} \quad T_1(\cos(\theta)) = \cos(\theta) \; .
$$
For $n \ge 1$, we have
\begin{align*}
T_{n+1}(\cos(\theta))
& = 2 \cos(\theta) T_n(\cos(\theta)) - T_{n-1}(\cos(\theta)) \\
& = 2 \cos(\theta) \cos(n \theta) - \cos((n-1)\theta)) \; ,
\end{align*}
where the last step follow by induction hypothesis.
It remains to show that the last expression equals $\cos((n+1)\theta)$.
This can be derived from the trigonometric formula
$$
\cos(\alpha \pm \beta) = \cos(\alpha) \cos(\beta) \mp \sin(\alpha) \sin(\beta) \; .
$$
By substituting $\alpha = n \theta$ and $\beta = \theta$, we get two equations
\begin{align*}
\cos((n+1) \theta) & = \cos(n \theta) \cos(\theta) - \sin(n \theta) \sin(\theta) \; , \\
\cos((n-1) \theta) & = \cos(n \theta) \cos(\theta) + \sin(n \theta) \sin(\theta) \; .
\end{align*}
Summing them yields
$$
\cos((n+1)\theta) + \cos((n-1) \theta) = 2 \cos(n \theta) \cos(\theta)
$$
which finishes the proof.

The fourth property has the similar proof as the third property. It suffices
to replace $\cos$ and $\sin$ with $\cosh$ and $\sinh$ respectively.

The fifth property follows from the third property. Indeed, for any $z \in [-1,1]$
there exists $\theta \in \R$ such that $\cos \theta = z$. Thus, $|T_n(z)| =
|T_n(\cos(\theta))| = |\cos(n\theta)| \le 1$.

The sixth property is equivalent to
$$
T_n(\cosh(\theta)) \ge 1 + n^2 (\cosh(\theta) - 1) \qquad \text{for all $\theta \ge 0$,}
$$
since $\cosh(\theta) = \frac{e^{\theta} + e^{-\theta}}{2}$ is an even continuous
function that maps $\R$ onto $[1,+\infty)$, is strictly decreasing on
$(-\infty,0]$, and is strictly increasing on $[0,\infty)$. Using the fourth
property the last inequality is equivalent to
$$
\cosh(n \theta) \ge 1 + n^2 (\cosh(\theta) - 1) \qquad \text{for all $\theta \ge 0$.}
$$
For $\theta = 0$, both sides are equal to $1$. Thus, it is sufficient to prove
that the derivative of the left hand side is greater or equal to the derivative
of the right hand side. Recalling that $[\cosh(\theta)]' = \sinh(\theta)$, this
means that we need to show that
$$
\sinh(n \theta) \ge n \sinh(\theta) \qquad \text{for all $\theta \ge 0$.}
$$
To prove this inequality we use the summation formula
$$
\sinh(\alpha + \beta) = \sinh(\alpha) \cosh(\beta) + \sinh(\beta) \cosh(\beta) \; .
$$
If $\alpha, \beta$ are non-negative then $\sinh(\alpha), \sinh(\beta)$ are
non-negative and $\cosh(\alpha), \cosh(\beta) \ge 1$. Hence,
$$
\sinh(\alpha + \beta) \ge \sinh(\alpha) + \sinh(\beta) \qquad \text{for any $\alpha, \beta \ge 0$.}
$$
This implies that (using induction on $n$) that $\sinh(n \theta) \ge n
\sinh(\theta)$ for all $\theta \ge 0$.

We verify the seventh property by induction on $n$.
For $n=0$ and $n=1$ the inequality trivially holds, since $\norm{T_0} = \norm{T_1} = 1$.
For $n \ge 1$, since $T_{n+1}(z) = 2zT_n(z) - T_{n-1}(z)$,
\begin{align*}
\norm{T_{n+1}}
& \le 2 \norm{T_n} + \norm{T_{n-1}} \\
& \le 2 (1 + \sqrt{2})^n + (1 + \sqrt{2})^{n-1} \\
& = (1 + \sqrt{2})^{n-1} (2 (1 + \sqrt{2}) + 1) \\
& = (1 + \sqrt{2})^{n-1} (3 + 2\sqrt{2}) \\
& = (1 + \sqrt{2})^{n-1} (1 + \sqrt{2})^2 \\
& = (1 + \sqrt{2})^{n+1} \; .
\qedhere
\end{align*}
\end{proof}

We are now ready to prove~\autoref{theorem:polynomial-approximation-1}.
Let $r = \left\lceil \log_2(2m) \right\rceil$ and $s = \left\lceil \sqrt{\frac{1}{\gamma}} \right\rceil$.
We define the polynomial $p:\R^d \to \R$ as
$$
p(x) = m + \frac{1}{2} - \sum_{i=1}^m \left( T_s(1 - \ip{v_i}{x}) \right)^r \; .
$$
It remains to show that $p$ has properties 1--5.

To verify the first property notice that if $x \in \R^d$ satisfies $\norm{x} \le
1$ and $\ip{v_i}{x} \ge \gamma$ then since $\norm{v_i} \le 1$ we have
$\ip{v_i}{x} \in [0,1]$. Thus, $T_s(1 - \ip{v_i}{x})$ and $\left( T_s(1 -
\ip{v_i}{x}) \right)^r$ lie in the interval $[-1,1]$. Therefore,
$$
p(x) \ge m + \frac{1}{2} - m \ge \frac{1}{2} \; .
$$

To verify the second property consider any $x \in \bigcup_{i=1}^m \left\{ x \in \R^d
~:~ \norm{x} \le 1, \ \ip{v_i}{x} \le - \gamma \right\}$. Clearly, $\norm{x} \le 1$
and there exists at least one $i \in \{1,2,\dots,m\}$ such that $\ip{v_i}{x} \le
- \gamma$. Therefore, $1 - \ip{v_i}{x} \ge 1 + \gamma$ and~\autoref{proposition:properties-of-chebyshev-polynomials} (part 6)
imply that
$$
T_s(1 - \ip{v_i}{x}) \ge 1 + s^2 \gamma \ge 2
$$
and thus
$$
\left( T_s(1 - \ip{v_i}{x}) \right)^r \ge 2^r \ge 2m \; .
$$
On the other hand for any $j \in \{1,2,\dots,m\}$, we have $\ip{v_j}{x} \in
[-1,1]$ and thus $1 - \ip{v_j}{x}$ lies in the interval $[0,2]$. According to
\autoref{proposition:properties-of-chebyshev-polynomials} (parts 5 and 6), $T_s(1 - \ip{v_j}{x})
\ge -1$. Therefore,
\begin{align*}
p(x) & = m + \frac{1}{2} - \left( T_s(1 - \ip{v_i}{x}) \right)^r - \sum_{\substack{j ~:~  1 \le j \le m \\ j \neq i}} \left( T_s(1 - \ip{v_j}{x}) \right)^r \\
& \le m + \frac{1}{2} - 2m + (m - 1) \le - \frac{1}{2} \; .
\end{align*}

The third property follows from the observation that the degree of $p$
is the same as the degree of any one of the terms
$\left( T_s(1 - \ip{v_i}{x}) \right)^r$ which is $r \cdot s$.

To prove the fourth property, we need to upper bound the norm of $p$.
Let $f_i(x) = 1 - \ip{v_i}{x}$, let $g_i(x) = T_s(1 - \ip{v_i}{x})$
and let $h_i(x) = (T_s(1 - \ip{v_i}{x}))^r$. We have
$$
\norm{f_i}^2 = 1 + \norm{v_i}^2 \le 1 + 1 = 2 \; .
$$
Let $T_s(z) = \sum_{j=0}^s c_j z^j$ be the expansion of $s$-th Chebyshev polynomial.
Then,
\begingroup
\allowdisplaybreaks
\begin{align*}
\norm{g_i}^2
& = \norm{ \sum_{j=0}^s c_j (f_i)^j }^2 \\
& \le (s + 1) \sum_{j=0}^s \norm{c_j (f_i)^j}^2 \quad \text{(by~part 3 of \autoref{lemma:properties-of-norm-of-polynomials})} \\
& = (s + 1) \sum_{j=0}^s (c_j)^2 \norm{(f_i)^j}^2 \\
& \le (s + 1) \sum_{j=0}^s (c_j)^2 j^j \norm{f_i}^{2j} \quad \text{(by~part 2 of \autoref{lemma:properties-of-norm-of-polynomials})} \\
& \le (s + 1) \sum_{j=0}^s (c_j)^2 j^j 2^{2j} \\
& \le (s + 1) s^s 2^{2s} \sum_{j=0}^s (c_j)^2 \\
& = (s + 1) s^s 2^{2s} \norm{T_s}^2 \\
& = (s + 1) s^s 2^{2s} (1 + \sqrt{2})^{2s} \quad \text{(by~part 7 of \autoref{proposition:properties-of-chebyshev-polynomials})} \\
& = (s + 1) \left(4(1+\sqrt{2})^2 s \right)^s \\
& \le \left(8(1+\sqrt{2})^2 s \right)^s \\
& \le \left(47 s \right)^s \; .
\end{align*}
\endgroup
where we used that $s+1 \le 2^s$ for any non-negative integer $s$.
Finally,
\begin{align*}
\norm{p}
& \le m + \frac{1}{2} + \sum_{i=1}^m \norm{(g_i)^r} \\
& = m + \frac{1}{2} + \sum_{i=1}^m \sqrt{\norm{(g_i)^r}^2} \\
& \le m + \frac{1}{2} + \sum_{i=1}^m \sqrt{r^{rs} \norm{g_i}^{2r}} \\
& \le m + \frac{1}{2} + m r^{rs/2} \left(47 s \right)^{rs/2} \\
& = m + \frac{1}{2} + m \left(47 rs \right)^{rs/2} \; .
\end{align*}
We can further upper bound the last expression by using that $m \le \frac{1}{2} 2^r$.
Since $r,s \ge 1$,
\begin{align*}
\norm{p}
& \le m + \frac{1}{2} + m \left(47 rs \right)^{rs/2} \\
& \le \frac{1}{2} 2^r + \frac{1}{2} + \frac{1}{2} 2^r \left(47 rs \right)^{rs/2} \\
& \le 2^r + \frac{1}{2} 2^r \left(47 rs \right)^{rs/2} \\
& = 2^r \left(1 + \frac{1}{2} \left(47 rs \right)^{rs/2} \right) \\
& = 2^r \left(47 rs \right)^{rs/2} \\
& \le 4^{rs/2} \left(47 rs \right)^{rs/2} \\
& \le \left(188 rs \right)^{rs/2} \; .
\end{align*}
Substituting for $r$ and $s$ finishes the proof.

\subsection{Proof of \autoref{theorem:polynomial-approximation-2}}
\label{section:proof-of-polynomial-approximation-2}

We define several univariate polynomials
\begin{align*}
P_n(z) & = (z - 1) \prod_{i=1}^n (z - 2^i)^2, \quad \text{for $n \ge 0$,} \\
A_{n,k}(z) & = (P_n(z))^k - (P_n(-z))^k,  \quad \text{for $n,k \ge 0$,} \\
B_{n,k}(z) & = - (P_n(z))^k - (P_n(-z))^k,  \quad \text{for $n,k \ge 0$.}
\end{align*}
We define the polynomial $q:\R^d \to \R$ as
$$
q(x) = \left[ \sum_{i=1}^m A_{s,r}\left( \frac{\ip{v_i}{x}}{\gamma} \right) \prod_{\substack{j ~:~ 1 \le j \le m \\ j \neq i}} B_{s,r} \left( \frac{\ip{v_j}{x}}{\gamma} \right) \right]
- \left(m - \frac{1}{2} \right) \prod_{j=1}^m B_{s,r} \left( \frac{\ip{v_j}{x}}{\gamma} \right) \; .
$$
Finally, we define $p(x) = 2^{-s(s+1)rm+1} q(x)$. We are going to show
that this polynomial $p$ satisfies the required properties.

For convenience we define univariate rational function
\begin{align*}
S_{n,k}(z) & = \frac{A_{n,k}(z)}{B_{n,k}(z)}, \quad \text{for $n,k \ge 0$,}
\end{align*}
and a multivariate rational function
$$
Q(x) = \left( \sum_{i=1}^m S_{s,r}\left( \frac{\ip{v_i}{x}}{\gamma} \right) \right) - \left(m - \frac{1}{2} \right) \; .
$$
It is easy to verify that
$$
q(x) = Q(x) \prod_{j=1}^m B_{s,r} \left( \frac{\ip{v_j}{x}}{\gamma} \right) \; .
$$

\begin{lemma}[Properties of $P_n$]
\label{lemma:properties-of-p-n}
\hspace{1cm} % Dummy space
\begin{enumerate}
\item If $z \in [0,1]$ then $P_n(-z) \le P_n(z) \le 0$.
\item If $z \in [1,2^n]$ then $0 \le 4P_n(z) \le -P_n(-z)$.
\item If $z \ge 0$ then $-P_n(-z) \ge 2^{n(n+1)}$.
\end{enumerate}
\end{lemma}

\begin{proof}
To prove the first part, note that $P_n(z)$ and $P_n(-z)$ are non-positive for
$z \in [0,1]$. We can write $\frac{P_n(z)}{P_n(-z)}$ as a product of $n+1$
non-negative fractions
$$
\frac{P_n(z)}{P_n(-z)} = \frac{1-z}{1+z} \prod_{i=1}^n \frac{(z+2^i)^2}{(z-2^i)^2} \; .
$$
The first part follows from the observation that each fraction is upper bounded
by $1$.

To prove the second part, notice that $P_n(z)$ is non-negative and $P_n(-z)$ is
non-positive for any $z \in [1,2^n]$. Now, fix $z \in [1,2^n]$ and let $j \in
\{1,2,\dots,n\}$ be such that $2^{j-1} \le z \le 2^j$. This implies that
$(z+2^{j})^2 \ge (2^j)^2 \ge 4 (z - 2^j)^2$. We can write
$\frac{P_n(z)}{-P_n(-z)}$ as a product of $n+1$ non-negative fractions
$$
\frac{P_n(z)}{-P_n(-z)}
= \frac{z-1}{z+1} \cdot \frac{(z-2^j)^2}{(z+2^j)^2} \prod_{\substack{i ~:~ 1 \le i \le n \\ i \neq j}} \frac{(z-2^i)^2}{(z+2^i)^2} \; .
$$
The second part follows from the observation that the second fraction is upper
bounded by $1/4$ and all other fractions are upper bounded by $1$.

The third part follows from
$$
-P_n(-z) = (1+z) \prod_{i=1}^n (z+2^i)^2 \ge \prod_{i=1}^n 2^{2i} = 2^{n(n+1)} \; .
$$
\end{proof}

\begin{lemma}[Properties of $S_{n,r}$ and $B_{n,r}$]
\label{lemma:properties-of-s-n-r}
Let $n,m$ be non-negative integers.
Let $r = 2 \left\lceil \frac{1}{4} \log_2(4m + 1) \right\rceil + 1$. Then,
\begin{enumerate}
\item If $z \in [1,2^n]$ then $S_{n,r}(z) \in [1,1+\frac{1}{2m}]$.
\item If $z \in [-2^n, -1]$ then $S_{n,r}(z) \in [-1-\frac{1}{2m}, -1]$.
\item If $z \in [-1,1]$ then $|S_{n,r}(z)| \le 1$.
\item If $z \in [-2^n,2^n]$ then $B_{n,r}(z) \ge \left(1 - \frac{1}{4m+1} \right) 2^{n(n+1)r}$.
\end{enumerate}
\end{lemma}

\begin{proof}
Note that $B_{n,r}(z)$ is an even function and $A_{n,r}(z)$ is an odd function.
Therefore, $S_{n,r}(z)$ is odd. Also notice that $r$ is an odd integer.

\begin{enumerate}
\item Observe that $S_{n,r}(z)$ can be written as
$$
S_{n,r}(z) = \frac{\displaystyle 1 + \left( - \frac{P_n(z)}{P_n(-z)}\right)^r}{\displaystyle 1 - \left( - \frac{P_n(z)}{P_n(-z)}\right)^r} = \frac{1 + c}{1 - c}
$$
where $c = \left( - \frac{P_n(z)}{P_n(-z)}\right)^r$. Since $z \in [1,2^n]$, by
part 2 of~\autoref{lemma:properties-of-p-n}, $c \in [0,\frac{1}{4^r}]$. Since
$r \ge \frac{1}{2} \log_2(4m+1)$, this means that $c \in [0,\frac{1}{4m+1}]$. Thus,
$S_{n,r}(z) = \frac{1+c}{1-c} \in [1,1 + \frac{1}{2m}]$.

\item Since $S_{n,r}(z)$ is odd, the statement follows from part 1.

\item Recall that $S_{n,r}(z)$ can be written as
$$
S_{n,r}(z) = \frac{1 + c}{1 - c}
$$
where $c = \left( - \frac{P_n(z)}{P_n(-z)}\right)^r$. If $z \in [0,1]$, by part
1 of~\autoref{lemma:properties-of-p-n} and the fact that $r$ is odd, $c
\in [-1,0]$, and thus, $S_{n,r}(z) = \frac{1+c}{1-c} \in [0,1]$. Since
$S_{n,r}(z)$ is odd, for $z \in [-1,0]$, $S_{n,r}(z) \in [-1,0]$.

\item Since $B_{n,r}(z)$ is even, we can without loss generality assume that $z \ge
0$. We consider two cases.

Case $z \in [0,1]$. Since $r$ is odd and $P_n(z)$ is non-positive,
\begin{align*}
B_{n,r}(z)
& = - (P_n(z))^r + \left(- P_{n}(-z)\right)^r \\
& \ge \left(- P_{n}(-z)\right)^r \ge 2^{n(n+1)r}  \\
& \ge 2^{n(n+1)r} \left( 1 - \frac{1}{4m+1} \right) \; .
\end{align*}
where the second last inequality follows from part 3 of \autoref{lemma:properties-of-p-n}.

Case $z \in [1,2^n]$. Since $r$ is odd,
\begin{align*}
B_{n,r}(z)
& = \left(- P_{n}(-z)\right)^r \left(1 - \left( - \frac{P_n(z)}{P_n(-z)}\right)^r \right) \\
& = \left(- P_{n}(-z)\right)^r (1 - c)
\end{align*}
where $c = \left( - \frac{P_n(z)}{P_n(-z)}\right)^r$. Since $z \in [1,2^n]$, by
part 2 of~\autoref{lemma:properties-of-p-n}, $c \in [0,\frac{1}{4^r}]$. By
the definition of $r$ that means that $c \in [0,\frac{1}{4m+1}]$. Thus,
\begin{align*}
B_{n,r}(z)
& \ge \left(- P_{n}(-z)\right)^r \left( 1 - \frac{1}{4m+1} \right) \\
& \ge 2^{n(n+1)r} \left( 1 - \frac{1}{4m+1} \right) \; .
\end{align*}
where the last inequality follows from part 3 of \autoref{lemma:properties-of-p-n}.
\end{enumerate}
\end{proof}

\begin{lemma}[Properties of $Q(x)$]
\label{lemma:properties-of-q}
The rational function $Q(x)$ satisfies
\begin{enumerate}
\item $Q(x) \ge \frac{1}{2}$ for all $\displaystyle x \in \bigcap_{i=1}^m \left\{ x \in \R^d ~:~ \norm{x} \le 1, \ \ip{v_i}{x} \ge \gamma \right\}$,
\item $Q(x) \le -\frac{1}{2}$ for all $\displaystyle x \in \bigcup_{i=1}^m \left\{ x \in \R^d ~:~ \norm{x} \le 1, \ \ip{v_i}{x} \le - \gamma \right\}$.
\end{enumerate}
\end{lemma}

\begin{proof}
To prove part 1, consider any $x \in \bigcap_{i=1}^m \left\{ x \in \R^d ~:~ \norm{x} \le 1, \
\ip{v_i}{x} \ge \gamma \right\}$. Then, $\frac{\ip{v_i}{x}}{\gamma} \in [1,
\frac{1}{\gamma}]$. By part 1 of \autoref{lemma:properties-of-s-n-r},
$S_{s,r}\left(\frac{\ip{v_i}{x}}{\gamma}\right) \in [1, 1 + \frac{1}{2m}]$ and
in particular $S_{s,r}\left(\frac{\ip{v_i}{x}}{\gamma}\right) \ge 1$. Thus,
\begin{align*}
Q(x)
& = \left( \sum_{i=1}^m S_{s,r}\left(\frac{\ip{v_i}{x}}{\gamma}\right) \right) - (m - 1/2) \\
& \ge m - (m - 1/2) \\
& = 1/2 \; .
\end{align*}

To prove part 2, consider any $x \in \bigcup_{i=1}^m \left\{ x \in \R^d ~:~
\norm{x} \le 1, \ \ip{v_i}{x} \le - \gamma \right\}$. Observe that
$\frac{\ip{v_i}{x}}{\gamma} \in [-\frac{1}{\gamma}, \frac{1}{\gamma}]$. Consider
$S_{s,r}\left(\frac{\ip{v_i}{x}}{\gamma}\right)$ for any $i \in
\{1,2,\dots,m\}$. Parts 1,2, and 3 of \autoref{lemma:properties-of-s-n-r}
and the fact $1/\gamma \le 2^s$ imply that
$S_{s,r}\left(\frac{\ip{v_i}{x}}{\gamma}\right) \le 1 +
\frac{1}{2m}$ for all $i \in \{1,2,\dots,m\}$. By the choice of $x$, there
exists $j \in \{1,2,\dots,m\}$ such that $\ip{v_j}{x} \le - \gamma$. Part 2 of
\autoref{lemma:properties-of-s-n-r} implies that
$S_{s,r}\left(\frac{\ip{v_j}{x}}{\gamma}\right) \in [-1-\frac{1}{2m},-1]$. Thus,
\begin{align*}
Q(x)
& = \left( \sum_{i=1}^m S_{s,r}\left( \frac{\ip{v_i}{x}}{\gamma} \right) \right) - \left(m - \frac{1}{2} \right) \\
& = S_{s,r}\left( \frac{\ip{v_j}{x}}{\gamma} \right) + \left( \sum_{\substack{i ~:~ 1 \le i \le m \\ i \neq j}} S_{s,r}\left( \frac{\ip{v_i}{x}}{\gamma} \right) \right) - \left(m - \frac{1}{2} \right) \\
& \le -1 + (m-1) \left( 1 + \frac{1}{2m} \right) - \left(m - \frac{1}{2} \right) \\
& \le -1/2 \; .
\qedhere
\end{align*}
\end{proof}

To prove parts 1 and 2 of \autoref{theorem:polynomial-approximation-2} first
note that part 4 of \autoref{lemma:properties-of-s-n-r} implies that for any $x$
such that $\norm{x} \le 1$, $B_{s,r}\left( \frac{\ip{v_i}{x}}{\gamma} \right)$
is positive. Thus $p(x)$ and $Q(x)$ have the same sign on the unit ball.
Consider any $x$ in either
$\displaystyle \bigcap_{i=1}^m \left\{ x \in \R^d ~:~ \norm{x} \le 1, \ \ip{v_i}{x} \ge \gamma \right\}$
or in
$\displaystyle \bigcup_{i=1}^m \left\{ x \in \R^d ~:~ \norm{x} \le 1, \ \ip{v_i}{x} \le - \gamma \right\}$.
\autoref{lemma:properties-of-q} states that $|Q(x)| \ge 1/2$ and the sign
depends on which of the two sets $x$ lies in. Since signs of $Q(x)$ and $p(x)$
are the same, it remains to show that $|p(x)| \ge \frac{1}{4} \cdot 2^{s(s+1)rm}$.
Indeed,
\begin{align*}
|p(x)|
& = 2^{-s(s+1)rm+1} \cdot |Q(x)| \prod_{j=1}^m B_{s,r} \left( \frac{\ip{v_j}{x}}{\gamma} \right) \\
& \ge 2^{-s(s+1)rm+1} \cdot |Q(x)| \left( 2^{s(s+1)r} \left( 1 - \frac{1}{4m+1} \right) \right)^m \\
& \ge |Q(x)| \ge \frac{1}{2} \quad \text{(\autoref{lemma:properties-of-q})} \; .
\end{align*}
where we used that $\left(1-\frac{1}{4m+1}\right)^m \ge e^{-\frac 1 4} \ge 1/2$.

%m^{s^2m/2}

To prove part 3 of \autoref{theorem:polynomial-approximation-2} note that
$\deg(P_s) = 2s+1$. Thus, $\deg(A_{s,r})$ and $\deg(B_{s,r})$ are at most
$(2s+1)r$. Therefore, $\deg(p) \le (2s+1) rm$.

It remains to prove part 4 of \autoref{theorem:polynomial-approximation-2}.
For any $i \in \{0,1,2,\dots,s\}$ and any $v \in \R^d$ such that $\norm{v} \le 1$
define multivariate polynomials
\begin{align*}
f_{i,v}(x) & = \frac{\ip{v}{x}}{\gamma} - 2^i \; , \\
q_v(x) & = P_s \left( \frac{\ip{v}{x}}{\gamma} \right) \; , \\
a_v(x) & = A_{s,r} \left( \frac{\ip{v}{x}}{\gamma} \right) \; , \\
b_v(x) & = B_{s,r} \left( \frac{\ip{v}{x}}{\gamma} \right) \; .
\end{align*}
Note that
$$
q(x) = \left[ \sum_{i=1}^m a_{v_i}(x) \prod_{\substack{j ~:~ 1 \le j \le m \\ j \neq i}} b_{v_j}(x) \right] - \left(m - \frac{1}{2} \right) \prod_{j=1}^n b_{v_j}(x) \; .
$$
We bound the norms of these polynomials. We have
$$
\norm{f_{i,v}}^2 = \norm{v}^2/\gamma^2 + 2^{2i} \le 2 \cdot 2^{2s} \; .
$$
where we used that $1/\gamma \le 2^s$ and $\norm{v} \le 1$.
Since $q_v(x) = f_{i,v}(\frac{\ip{v}{x}}{\gamma}) \prod_{i=1}^s \left(f_{i,v}(\frac{\ip{v}{x}}{\gamma})\right)^2$,
using part 1 of \autoref{lemma:properties-of-norm-of-polynomials} we upper bound the norm of $q_v$
as
\begin{align*}
\norm{q_v}^2
& \le (2s+1)^{2s+1} \norm{f_{0,v}}^2 \prod_{i=1}^s \norm{f_{i,v}}^4 \\
& \le  (2s+1)^{2s+1} (2 \cdot 2^{2s})^{2s + 1} \; .
\end{align*}
Using parts 3 and 2 of \autoref{lemma:properties-of-norm-of-polynomials} we upper bound the norm of $a_v$ as
\begin{align*}
\norm{a_v}^2
& \le 2\norm{(q_v)^r}^2 + 2\norm{(q_{-v})^r}^2 \\
& \le 2 r^{r(2s+1)} (\norm{q_v}^{2})^r + 2 r^{r(2s+1)} (\norm{q_{-v}}^{2})^r \\
& \le 4 r^{r(2s+1)} \left((2s+1)^{2s+1} (2 \cdot 2^s)^{2s + 1} \right)^{r} \\
& = 4 \left(2^{2s} r (4s+2) \right)^{(2s+1)r} \; .
\end{align*}
The same upper bound holds for $\norm{b_v}^2$. Therefore,
\begin{align*}
\norm{q}
& \le \left[ \sum_{i=1}^m \norm{a_{v_i} \prod_{\substack{j ~:~ 1 \le j \le m \\ j \neq i}} b_{v_j}} \right] + \left(m - \frac{1}{2} \right) \norm{\prod_{j=1}^m b_{v_j}} \\
& \le \left[ \sum_{i=1}^m m^{(s+1/2)rm} \norm{a_{v_i}} \prod_{\substack{j ~:~ 1 \le j \le m \\ j \neq i}} \norm{b_{v_j}} \right] \\
& \qquad + \left(m - \frac{1}{2} \right) m^{(s+1/2)rm} \prod_{j=1}^m \norm{b_{v_j}} \\
& \le (2m-1/2) m^{(s+1/2)rm} \left(4 \left(2^{2s} r (4s+2) \right)^{(2s+1)r} \right)^{m/2} \\
& = (2m-1/2) 2^m \cdot \left(2^{2s} rm (4s+2) \right)^{(s+1/2)rm} \; .
\end{align*}
Finally,
$\norm{p} = 2^{-s(s+1)rm+1} \norm{q} \leq (4m-1) 2^m \cdot \left(2^s rm (4s+2) \right)^{(s+1/2)rm}$. The theorem follows.

\section{Proof of Theorem~\ref{theorem:margin-transformation}}
\label{section:proof-of-theorem-margin-transformation}

\begin{proof}[Proof of Theorem~\ref{theorem:margin-transformation}]
Since the examples $(x_1, y_1)$, $(x_2, y_2)$, $\dots$, $(x_T, y_T)$ are weakly
linearly separable with margin $\gamma$,, there are vectors $w_1, w_2, \dots, w_K$
satisfying \eqref{equation:weak-linear-separability-1} and
\eqref{equation:weak-linear-separability-2}.

Fix any $i \in \{1,2,\dots,K\}$. Consider the $K-1$ vectors $(w_i - w_j)/2$ for
$j \in \{1,2,\dots,K\} \setminus \{i\}$. Note that the vectors have norm at most
$1$. We consider two cases regarding the relationship between $\gamma_1$ and
$\gamma_2$.

\paragraph{Case 1: $\gamma_1 \geq \gamma_2$.} In this case, Theorem~\ref{theorem:polynomial-approximation-1}
implies that there exist a multivariate polynomial $p_i:\R^d \to \R$,
\begin{align*}
\deg(p_i) & = \lceil \log_2(2K-2) \rceil \cdot \left\lceil \sqrt{\frac{2}{\gamma}} \right\rceil \; ,
\end{align*}
such that all examples $x$ in $R_i^+$ (resp. $R_i^-$) satisfy $p_i(x) \geq 1/2$
(resp. $p_i(x) \leq -1/2$).
Therefore, for all $t=1,2,\dots,T$, if $y_t = i$ then $p_i(x_t) \ge 1/2$,
 and if $y_t \neq i$ then $p_i(x_t) \le -1/2$, and
\[
\norm{p_i} \le
\left(188 \lceil \log_2(2K-2) \rceil \cdot \left \lceil \sqrt{\frac{2}{\gamma}} \right \rceil \right)^{\frac{1}{2} \lceil \log_2(2K-2) \rceil
\cdot \left \lceil \sqrt{\frac{2}{\gamma}} \right \rceil} \; .
\]
By \autoref{lemma:norm-bound}, there exists $c_i \in \ell_2$ such that
$\ip{c_i}{\phi(x)} = p_i(x)$, and
\[
\norm{c_i}_{\ell_2} \le
\left(376 \lceil \log_2(2K-2) \rceil \cdot \left \lceil \sqrt{\frac{2}{\gamma}} \right \rceil \right)^{\frac{1}{2} \lceil \log_2(2K-2) \rceil
\cdot \left \lceil \sqrt{\frac{2}{\gamma}} \right \rceil} \; .
\]
Define vectors $u_i \in \ell_2$ as
\[
u_i = \frac{1}{\sqrt{K}}
\cdot \frac{c_i}{\left(376 \lceil \log_2(2K-2) \rceil \cdot \left \lceil \sqrt{\frac{2}{\gamma}} \right \rceil \right)^{\frac{1}{2} \lceil \log_2(2K-2) \rceil
\cdot \left \lceil \sqrt{\frac{2}{\gamma}} \right \rceil}} \; .
\]

Then, $\norm{u_1}^2 + \norm{u_2}^2 + \dots + \norm{u_K}^2 \le 1$.
Furthermore, for all $t=1,2,\dots,T$, $\ip{u_{y_t}}{\phi(x_t)} \ge \gamma_1$
and for all $j \in \{1,2,\dots,K\} \setminus \{y_t\}$,
$\ip{u_j}{\phi(x_t)} \le - \gamma_1$. In other words,
$(\phi(x_1), y_1), (\phi(x_2), y_2), \dots, (\phi(x_T), y_T)$ are
strongly linearly separable with margin $\gamma_1 = \max\{\gamma_1, \gamma_2\}$.

\paragraph{Case 2: $\gamma_1 < \gamma_2$.} In this case, Theorem~\ref{theorem:polynomial-approximation-2}
implies that there exist a multivariate polynomial $q_i:\R^d \to \R$,
\begin{align*}
\deg(q_i) & = (2s+1) r(K-1) \; ,
\end{align*}
such that all examples $x$ in $R_i^+$ (resp. $R_i^-$) satisfy $q_i(x) \geq 1/2$
(resp. $q_i(x) \leq -1/2$), and
\begin{align*}
  \norm{q_i} \le (4K-5) 2^{K-1} \cdot \left(2^{s} r(K-1) (4s+2)\right)^{(s+1/2)r(K-1)} \; .
\end{align*}
Recall that here,
\[
r = 2 \left\lceil \frac{1}{4} \log_2(4K - 3) \right\rceil + 1 \quad \text{and} \quad s = \left \lceil \log_2(1/\gamma) \right \rceil \; .
\]

Therefore, for all $t=1,2,\dots,T$, if $y_t = i$ then $q_i(x_t) \ge 1/2$,
 and if $y_t \neq i$ then $q_i(x_t) \le -1/2$.

By \autoref{lemma:norm-bound}, there exists $c'_i \in \ell_2$ such that
$\ip{c'_i}{\phi(x)} = p_i(x)$, and
\[
  \norm{c'_i}_{\ell_2} \le (4K-5) 2^{K-1} \cdot \left(2^{s+1} r(K-1) (4s+2) \right)^{(s+1/2)r(K-1)} \; .
\]
Define vectors $u_i' \in \ell_2$ as
\begin{align*}
u_i' = \frac{c_i'  \cdot \left(2^{s+1} r(K-1) (4s+2) \right)^{-(s+1/2)r(K-1)}}{\sqrt{K} (4K-5) 2^{K-1}} \; .
\end{align*}
Then,
$\norm{u_1'}^2 + \norm{u_2'}^2 + \dots + \norm{u_K'}^2 \le 1$.
Furthermore, for all $t=1,2,\dots,T$,
$\ip{u'_{y_t}}{\phi(x_t)} \ge \gamma_2$
and for all $j \in \{1,2,\dots,K\} \setminus \{y_t\}$,
$\ip{u'_j}{\phi(x_t)} \le - \gamma_2$. In other words,
$(\phi(x_1), y_1), (\phi(x_2), y_2), \dots, (\phi(x_T), y_T)$ are
strongly linearly separable with margin $\gamma_2 = \max\{\gamma_1, \gamma_2\}$.
%The proof for the case of $\gamma_1 < \gamma_2$
%is symmetric.

In summary, the examples are strongly
linearly separable with margin $\gamma' = \max\{\gamma_1, \gamma_2\}$.
Finally, observe that for any $t=1,2,\dots,T$,
\[
k(x_t,x_t) = \frac{1}{1 - \frac{1}{2} \norm{x_t}^2} \le 2 \; .
\qedhere
\]
\end{proof}

\section{Supplementary Materials for Section~\ref{section:experiments}}
\label{section:supp-to-experiment}

Figures~\ref{fig:banditron-points},~\ref{fig:linearova-points}, and~\ref{fig:rationalova-points} show the final decision boundaries learned by each algorithm on the two datasets (Figures~\ref{figure:number-of-mistakes-strongly-separable-dataset} and~\ref{figure:number-of-mistakes-weakly-separable-dataset}), 
after $T = 5 \times 10^6$ rounds. We used the version of Banditron
with exploration rate of 0.02, which explores the most.

\begin{figure}[h!]
    \centering
    \begin{subfigure}[b]{0.25\textwidth}
        \captionsetup{justification=centering}
        \begin{center}
        \includegraphics[width=\textwidth, trim={0, 0cm, 0, 0}, clip]{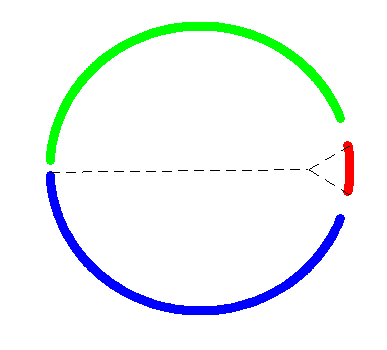}
        \caption{Strongly separable case}
        \end{center}
    \end{subfigure}
    \begin{subfigure}[b]{0.25\textwidth}
        \captionsetup{justification=centering}
        \centering
        \includegraphics[width=\textwidth, trim={0, 0cm, 0, 0}, clip]{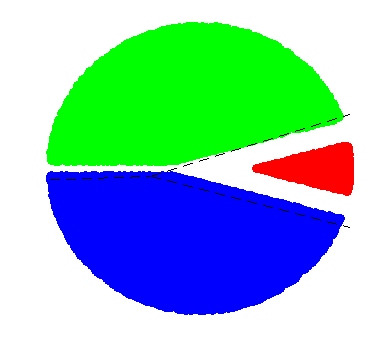}
         \caption{Weakly separable case}
    \end{subfigure}
    \captionsetup{justification=centering}
    \caption{\textsc{Banditron}'s final decision boundaries}
     \label{fig:banditron-points}
\end{figure}

\begin{figure}[h!]
    \centering
    \begin{subfigure}[b]{0.25\textwidth}
        \captionsetup{justification=centering}
        \begin{center}
        \includegraphics[width=\textwidth, trim={0, 0cm, 0, 0}, clip]{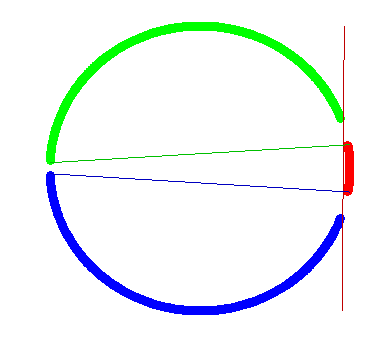}
        \caption{Strongly separable case}
        \end{center}
    \end{subfigure}
    \begin{subfigure}[b]{0.25\textwidth}
        \captionsetup{justification=centering}
        \centering
        \includegraphics[width=\textwidth, trim={0, 0cm, 0, 0}, clip]{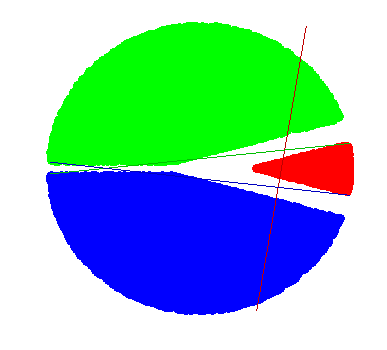}
         \caption{Weakly separable case}
    \end{subfigure}
    \captionsetup{justification=centering}
    \caption{Algorithm~\ref{algorithm:algorithm-for-strongly-linearly-separable-examples}'s final decision boundaries}
    \label{fig:linearova-points}
\end{figure}
%Our Algorithm with linear kernel (
%To better understand how each algorithm works,
%Our Algorithm with rational kernel (

\begin{figure}[h!]
    \centering
    \begin{subfigure}[b]{0.25\textwidth}
        \captionsetup{justification=centering}
        \begin{center}
        \includegraphics[width=\textwidth, trim={0, 0cm, 0, 0}, clip]{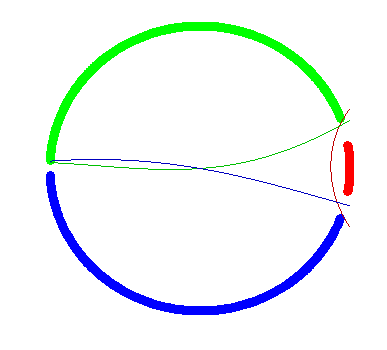}
        \caption{Strongly separable case}
        \end{center}
    \end{subfigure}
    \begin{subfigure}[b]{0.25\textwidth}
        \captionsetup{justification=centering}
        \centering
        \includegraphics[width=\textwidth, trim={0, 0cm, 0, 0}, clip]{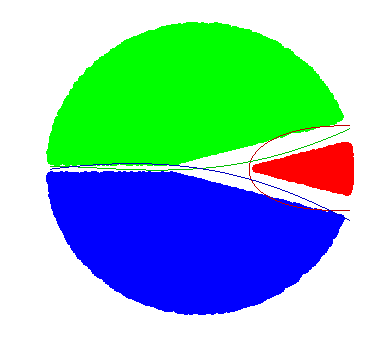}
         \caption{Weakly separable case}
    \end{subfigure}
    \captionsetup{justification=centering}
    \caption{Algorithm~\ref{algorithm:kernelized} (with rational kernel)'s final decision boundaries}
    \label{fig:rationalova-points}
\end{figure}

\section{Nearest neighbor algorithm}
\label{section:nearest-neighbor-algorithm}

\begin{algorithm}[H]
\SetAlgoLined
\LinesNumbered
\setcounter{AlgoLine}{0}
\caption{\textsc{Nearest-Neighbor Algorithm}
\label{algorithm:nearest-neighbor}
}
\textbf{Require:} Number of classes $K$, number of rounds $T$. \\
\textbf{Require:} Inner product space $(V,\ip{\cdot}{\cdot})$. \\
\nl Initialize $S \gets \emptyset$ \\
\nl \For{$t=1,2,\ldots,T$:}{
\nl  \If{$\min_{(x,y) \in S} \norm{x_t - x} \le \gamma$}{
\nl    Find nearest neighbor \\
       \qquad \qquad $(\widetilde{x}, \widetilde{y}) = \argmin_{(x,y) \in S} \norm{x_t - x}$ \\
\nl	  Predict $\widehat{y}_t = \widetilde{y}$
  }
\nl \Else{
\nl	  Predict $\widehat y_t \sim \text{Uniform}(\{1,2,\dots,K\})$
      \label{line:nearest-neighbor-explore} \\
\nl	  Receive feedback $z_t = \indicator{\widehat{y}_t \neq y_t}$ \\
\nl	  \If{$z_t = 0$}{
\nl      $S \gets S \cup \cbr{(x_t, \widehat{y}_t)}$
    }
  }
}
\end{algorithm}

In this section we analyze \textsc{Nearest-Neighbor Algorithm} shown as
Algorithm~\ref{algorithm:nearest-neighbor}. The algorithm is based on the obvious
idea that, under the weak linear separability assumption, two examples that are close to each
other must have the same label. The lemma below formalizes this intuition.

\begin{lemma}[Non-separation lemma]
\label{lemma:non-separation-lemma}
Let $(V,\ip{\cdot}{\cdot})$ be a vector space, $K$ be a positive integer and let
$\gamma$ be a positive real number. Suppose $(x_1,y_1), (x_2,y_2), \dots, (x_T,
y_T) \in V \times \{1,2,\dots,K\}$ are labeled examples that are weakly linearly
separable with margin $\gamma$. For $i$, $j$ in $\cbr{1,2,\dots,T}$, if
$\norm{x_i - x_j}_2 \le \gamma$ then $y_i = y_j$.
\end{lemma}

\begin{proof}
Suppose for the sake on contradiction that $y_i \neq y_j$. By
Definition~\ref{definition:linear-separability}, there exists
vectors $w_1, \ldots, w_K$ such that
conditions~\eqref{equation:weak-linear-separability-1}
and~\eqref{equation:weak-linear-separability-2} are satisfied.

Specifically,
\begin{align*}
\ip{w_{y_i} - w_{y_j}}{x_i} & \ge \gamma \; , \\
\ip{w_{y_j} - w_{y_i}}{x_j} & \ge \gamma \; .
\end{align*}
This implies that
$$
\ip{w_{y_i} - w_{y_j}}{x_i - x_j} \ge 2\gamma \; .
$$

On the other hand,
$$
\ip{w_{y_i} - w_{y_j}}{x_i - x_j} \le \norm{w_{y_i} - w_{y_j}} \cdot \norm{x_i - x_j} \le \sqrt{2} \gamma
$$
where the first inequality is from Cauchy-Schwartz inequality, the second
inequality is from that $\norm{w_{y_i} - w_{y_j}} \le \sqrt{2(\norm{w_{y_i}}^2 +
\norm{w_{y_j}}^2)} \leq \sqrt{2}$ and our assumption on $x_i$ and $x_j$.
Therefore, we reach a contradiction.
\end{proof}

We also need to define several notions. A subset $S \subseteq \R^d$ is called a
$\gamma$-packing if for any $x,x' \in S$ such that $x \neq x'$ we have $\norm{x -
x'} > \gamma$. The following lemma is standard. Also recall that $\B(x,R) = \{ x'
\in \R^d ~:~ \norm{x' - x} \le R \}$ denotes the closed ball of radius $R$
centered a point $x$. For set $S \subseteq \R^d$, denote by $\Vol(S)$ the volume of
$S$.

\begin{lemma}[Size of $\gamma$-packing]
\label{lemma:size-of-packing}
Let $\gamma$ and $R$ be positive real numbers.
If $S \subseteq \B(\zero,R) \subseteq \R^d$ is a $\gamma$-packing then
$$
|S| \le \left( \frac{2R}{\gamma} + 1 \right)^d \; .
$$
\end{lemma}

\begin{proof}
If $S$ is a $\gamma$-packing then $\{ \B(x,\gamma/2) ~:~ x \in S \}$
is a collection of disjoint balls of radius $\gamma$ that fit into $\B(\zero,R + \gamma/2)$.
Thus,
$$
|S| \cdot \Vol(\B(\zero, \gamma/2)) \le \Vol(\B(\zero,R + \gamma/2))
$$
Hence,
\begin{align*}
|S|
 \le \frac{\Vol(\B(\zero,R + \gamma/2))}{\Vol(\B(\zero, \gamma/2))} 
 = \left( \frac{R + \gamma/2}{\gamma/2} \right)^d
 = \left( \frac{2R}{\gamma} + 1 \right)^d \; .
\end{align*}
\end{proof}

\begin{theorem}[Mistake upper bound for \textsc{Nearest-Neighbor Algorithm}]
\label{theorem:mistake-bound-for-nearest-neighbor-algorithm}
Let $K$ and $d$ be positive integers and let $\gamma,R$ be a positive real
numbers. Suppose $(x_1,y_1), \ldots, (x_T, y_T) \in
\R^d \times \{1,2,\dots,K\}$ are labeled examples that are weakly linearly
separable with margin $\gamma$ and satisfy $\norm{x_1}, \norm{x_2}, \dots,
\norm{x_T} \le R$. Then, the expected number of mistakes made by
Algorithm~\ref{algorithm:nearest-neighbor} is at most
$$
(K-1) \left( \frac{2R}{\gamma} + 1\right)^d \; .
$$
\end{theorem}

\begin{proof}
Let $M$ be the number of mistakes made by the algorithm. Let $b_t$ be the
indicator that line~\ref{line:nearest-neighbor-explore} is executed at time step
$t$, i.e. we fall into the ``else'' case. Note that if $b_t = 0$, then by
Lemma~\ref{lemma:non-separation-lemma}, the prediction $\widehat{y}_t$ must equal
$y_t$, i.e. $z_t = 0$. Therefore, $M = \sum_{t=1}^T z_t = \sum_{t=1}^T b_t z_t$.
Let $U = \sum_{t=1}^T b_t (1-z_t)$. Clearly, $|S| = U$. Since $S \subseteq \B(\zero, R)$
is a $\gamma$-packing, $U = |S| \le (\frac{2R}{\gamma} + 1)^d$.

Note that when $b_t = 1$, $\widehat{y}_t$ is chosen uniformly at random, we have
$$
\Exp[ z_t ~|~ b_t = 1] = \frac{K-1}{K} \; .
$$
Therefore,
$$
\Exp[M] = \Exp \left[ \sum_{t=1}^T b_t z_t \right] = \frac{K-1}{K} \Exp \left[ \sum_{t=1}^T b_t \right] \; .
$$
On the other hand,
$$
\Exp[U] = \Exp \left[ \sum_{t=1}^T b_t (1-z_t) \right] = \frac 1 K \Exp \left[\sum_{t=1}^T b_t \right] \; .
$$
Therefore,
$$
\Exp[M] = (K-1) \Exp[U] \leq (K-1) \left(\frac{2R}{\gamma} + 1 \right)^d \; .
$$
\end{proof}

\section{NP-hardness of the weak labeling problem}
\label{section:np-hardness-of-weak-labeling-problem}

Any algorithm for the bandit setting collects information in the form of so
called \emph{strongly labeled} and \emph{weakly labeled} examples.
Strongly-labeled examples are those for which we know the class label. Weakly
labeled example is an example for which we know that class label can be anything
except for a particular one class.

A natural strategy for each round is to find vectors $w_1, w_2, \dots, w_K$ that
linearly separate the examples seen in the previous rounds and use the vectors
to predict the label in the next round. More precisely, we want to find both the
vectors $w_1, w_2, \dots, w_K$ and label for each example consistent with its
weak and/or strong labels such that $w_1, w_2, \dots, w_K$ linearly separate the
labeled examples. We show this problem is NP-hard even for $K=3$.

Clearly, the problem is at least as hard as the decision version of the problem
where the goal is to determine if such vectors and labeling exist. We show that
this problem is NP-complete.

We use symbols $[K]=\{1,2,\dots,K\}$ for strong labels and $[\overline{K}]
=\{\overline{1},
\overline{2}, \dots, \overline{K}\}$ for weak labels. Formally, the weak
labeling problem can be described as below:
\begin{figure}[H]
\begin{framed}
\begin{center}
    \textbf{Weak Labeling}
\end{center}
\textbf{Given:} Feature-label pairs $(x_1, y_1)$, $(x_2, y_2)$, \dots, $(x_T, y_T)$ in $\{0,1\}^d \times \{1,2,\dots, K, \overline{1}, \overline{2}, \dots, \overline{K}\}$. \\
\textbf{Question:} Do there exist $ w_1, w_2, \dots, w_K \in \R^d$ such that for all $t=1,2,\dots,T$,
\begin{align*}
& y_t \in [K] \Longrightarrow \forall i \in [K] \setminus \{y_t\} \quad \ip{w_{y_t}}{x_t}  > \ip{w_i}{x_t} \; , \\
& \text{and} \\
& y_t \in [\overline{K}] \Longrightarrow \exists i \in [K] \quad \ip{w_i}{x_t} > \ip{w_{\overline{y_t}}}{x_t} \; ?
\end{align*}
\end{framed}
\end{figure}

The hardness proof is based on a reduction from the set splitting problem, which
is proven to be NP-complete by Lov\'asz \cite{Garey-Johnson-1979}, to our weak
labeling problem. The reduction is adapted from \cite{Blum-Rivest-1993}.
\begin{figure}[H]
\begin{framed}
\begin{center}
    \textbf{Set Splitting}
\end{center}
\textbf{Given:} A finite set $S$ and a collection $C$ of subsets $c_i$ of $S$. \\
\textbf{Question:} Do there exist disjoint sets $S_1$ and $S_2$ such that $S_1 \cup S_2 = S$ and $\forall i, c_i\not\subseteq S_1$ and $c_i\not\subseteq S_2$?
\end{framed}
\end{figure}

Below we show the reduction. Suppose we are given an instance of the set
splitting problem
\begin{align*}
S = \{1, 2, \dots, N\} \; ,
C = \{c_1, c_2, \dots, c_M\} \; .
\end{align*}

We create the weak labeling instance as follows. Let $d=N+1$ and $K=3$.
Define $\zero$ as the zero vector $(0,\dots,0)\in \R^N$ and $\e_i$ as the
$i$-th standard vector $(0,\dots, 1, \dots, 0)\in \R^N$. Then we include all
the following feature-label pairs:
\begin{itemize}
\item Type 1: $(x,y)=((\zero,1), 3)$,
\item Type 2: $(x,y)=((\e_i,1), \overline{3})$ for all $i \in \{1,2,\dots,N\}$,
\item Type 3: $(x,y)=\left(\left(\sum_{i\in c_j} \e_i, 1\right), 3\right)$ for all $j \in \{1,2,\dots,M\}$.
\end{itemize}

For example, if we have $S=\{1,2,3\}$, $C=\{c_1, c_2\}$, $c_1 = \{1,2\}$,
$c_2=\{2,3\}$, then we create the weak labeling sample set as:
\[
\{
((0,0,0,1),3), ((1,0,0,1),\overline{3}), ((0,1,0,1),\overline{3}),
((0,0,1,1),\overline{3}), ((1,1,0,1),3), ((0,1,1,1),3)
\} \; .
\]
The following lemma shows that answering this weak labeling problem is
equivalent to answering the original set splitting problem.

\begin{lemma}
Any instance of the set splitting problem is a YES instance if and only if the
corresponding instance of the weak labeling problem (as described above) is a
YES instance.
\end{lemma}

\begin{proof}
$(\Longrightarrow)$ Let $S_1, S_2$ be the solution of the set splitting problem. Define
$$
w_1 = \left(a_1, a_2, \cdots, a_N, -\frac{1}{2}\right),
$$
where for all $i \in \{1,2,\dots,N\}$, $a_i=1$ if $i\in S_1$ and $a_i=-N$ if
$i\notin S_1$. Similarly, define
$$
w_2 = \left(b_1, b_2, \cdots, b_N, -\frac{1}{2}\right),
$$
where for all $i \in \{1,2,\dots,N\}$, $b_i=1$ if $i \in S_2$ and $b_i=-N$ if
$i\notin S_2$. Finally, define
$$
w_3 = (0,0,\cdots, 0),
$$
the zero vector. To see this is a solution for the weak labeling problem, we
verify separately for Type 1-3 samples defined above. For Type 1 sample, we have
$$
\ip{w_3}{x} = 0 > -\frac{1}{2} = \ip{w_1}{x}=\ip{w_2}{x}.
$$
For a Type 2 sample that corresponds to index $i$, we have either $i\in S_1$ or
$i\in S_2$ because $S_1\cup S_2 = \{1,2,\dots,N\}$ is guaranteed. Thus, either
$a_i=1$ or $b_i=1$. If $a_i=1$ is the case, then
$$
\ip{w_1}{x} = a_i - \frac{1}{2} = \frac{1}{2} > 0 = \ip{w_3}{x};
$$
similarly if $b_i=1$, we have $\ip{w_2}{x}>\ip{w_3}{x}$. \\ For a Type 3 sample
that corresponds to index $j$, Since $c_j \not\subset S_1$, there exists some
$i'\in c_j$ and $i'\notin S_1$. Thus we have $x_{i'}=1$, $a_{i'}=-N$, and
therefore
\begin{align*}
\ip{w_1}{x}
& = a_{i'}x_{i'} + \sum_{i\in \{1,2,\dots,N\} \setminus \{i'\}} a_ix_i - \frac{1}{2} \\
& \le -N + (N-1)-\frac{1}{2} < 0 = \ip{w_3}{x} \; .
\end{align*}
Because $c_j \not\subset S_2$ also holds, we also have
$\ip{w_2}{x}<\ip{w_3}{x}$. This direction is therefore proved. \\
\ \\
$(\Longleftarrow)$ Given the solution $w_1, w_2, w_3$ of the weak labeling problem, we define
\begin{align*}
S_1 &= \left\{i \in \{1,2,\dots,n\} ~:~ \ip{w_1-w_3}{(\e_i, 1)} > 0 \right\}, \\
S_2 &= \left\{i \in \{1,2,\dots,n\} ~:~ \ip{w_2-w_3}{(\e_i, 1)} > 0 \text{\ and\ } i \notin S_1 \right\}.
\end{align*}
It is not hard to see $S_1 \cap S_2 = \emptyset$ and $S_1\cup S_2 =
\{1,2,\dots,N\}$. The former is because $S_2$ only includes elements that are
not in $S_1$. For the latter, note that $(\e_i, 1)$ is the feature vector for
Type 2 samples. Because Type 2 samples all have label $\overline{3}$, for any $i
\in \{1,2,\dots,N\}$, one of the following must hold: $\ip{w_1-w_3}{(\e_i,
1)}>0$ or $\ip{w_2-w_3}{(\e_i, 1)}>0$. This implies $i\in S_1$ or $i\in S_2$.

Now we show $\forall j$, $c_j \not\subset S_1$ and $c_j \not\subset S_2$ by
contradiction. Assume there exists some $j$ such that $c_j \subset S_1$. By our
definition of $S_1$, we have $\ip{w_1-w_3}{(\e_i, 1)} > 0$ for all $i\in c_j$.
Therefore,
\begin{align*}
\sum_{i\in c_j} \ip{w_1-w_3}{\left(\e_i, 1\right)}
 = \ip{w_1-w_3}{\left(\sum_{i\in c_j} \e_i, |c_j|\right)} > 0.
\end{align*}
Because Type 1 sample has label $3$, we also have
$$
\ip{w_1-w_3}{\left(\zero, 1\right)} < 0.
$$
Combining the above two inequalities, we get
\begin{align*}
& \ip{w_1-w_3}{\left(\sum_{i\in c_j}\e_i, 1\right)}
 = \ip{w_1-w_3}{\left(\sum_{i\in c_j}\e_i, |c_j|\right)} - (|c_j|-1)\ip{w_1-w_3}{\left(\zero, 1\right)}  > 0 \; .
\end{align*}
Note that $\left(\sum_{i\in c_j}\e_i, 1\right)$ is a feature vector for Type 3
samples. Thus the above inequality contradicts that Type 3 samples have label 3.
Therefore, $c_j \not\subset S_1$. If we assume there exists some $c_j \subset
S_2$, same arguments apply and also lead to contradiction.
\end{proof}

\section{Mistake lower bound for ignorant algorithms}
\label{section:mistake-lower-bound-for-ignorant-algorithms}

In this section, we prove a mistake lower bound for a family of algorithms
called \textit{ignorant algorithms}. Ignorant algorithms ignore the examples on
which they make mistakes. This assumption seems strong, but as we will explain
below, it is actually natural, and several recently proposed bandit linear
classification algorithms that achieve $\sqrt{T}$ regret bounds belong to this
family, e.g., SOBA~\citep{Beygelzimer-Orabona-Zhang-2017},
OBAMA~\citep{Foster-Kale-Luo-Mohri-Sridharan-2018}. Also,
\textsc{Nearest-Neighbor Algorithm} (Algorithm~\ref{algorithm:nearest-neighbor})
presented in Appendix~\ref{section:nearest-neighbor-algorithm} is an ignorant
algorithm.

Under the assumption that the examples lie in in the unit ball of $\R^d$ and are
weakly linearly separable with margin $\gamma$, we show that any ignorant
algorithm must make at least $\Omega \left( \left(\frac{1}{160
\gamma}\right)^{(d-2)/4} \right)$ mistakes in the worst case. In other words, an
algorithm that achieves a better mistake bound cannot ignore examples on which
it makes a mistake and it must make a meaningful update on such examples.

To formally define ignorant algorithms, we define the conditional distribution
from which an algorithm draws its predictions. Formally, given an algorithm
$\calA$ and an adversarial strategy, we define
\[
p_t(y|x) =
\Pr[y_t = y ~|~ (x_1, y_1), (x_2, y_2) \dots, (x_{t-1}, y_{t-1}), x_t = x] \; .
\]
In other words, in any round $t$, conditioned on the past $t-1$ rounds, the
algorithm $\calA$ chooses $y_t$ from probability distribution $p_t(\cdot|x_t)$.
Formally, $p_t$ is a function $p:\{1,2,\dots,K\} \times \R^d \to [0,1]$
such that $\sum_{y=1}^K p_t(y|x) = 1$ for any $x \in \R^d$.

\begin{definition}[Ignorant algorithm]
An algorithm $\calA$ for \textsc{Online Multiclass Linear Classification with
Bandit Feedback} is called \emph{ignorant} if for every $t=1,2,\dots,T$,
$p_t$ is determined solely by the sequence
$(x_{a_1}, y_{a_1})$,$(x_{a_2}, y_{a_2})$, $\dots$, $(x_{a_n}, y_{a_n})$
of labeled examples
from the rounds $1 \le a_1 < a_2 < \dots < a_n < t$ in which
the algorithm makes a correct prediction.
\end{definition}

An equivalent definition of an ignorant algorithm is that the memory state of
the algorithm does not change after it makes a mistake. Equivalently,
the memory state of an ignorant algorithm is completely determined
by the sequence of labeled examples on which it made correct prediction.

To explain the definition, consider an ignorant algorithm $\calA$. Suppose that
on a sequence of examples $(x_1, y_1)$, $(x_2, y_2)$, $\dots$, $(x_{t-1}, y_{t-1})$
generated by some adversary the algorithm $\calA$ makes correct predictions in
rounds $a_1, a_2, \dots, a_n$ where $1 \le a_1 < a_2 < \dots < a_n < t$ and
errors on rounds $\{1,2,\dots,t-1\} \setminus \{a_1, a_2, \dots, a_n\}$. Suppose
that on another sequence of examples $(x_1', y_1'), (x_2', y_2'), \dots,
(x_{s-1}', y_{s-1}')$ generated by another adversary the algorithm $\calA$ makes
correct predictions in rounds $b_1, b_2, \dots, b_n$ where $1 \le b_1 < b_2 <
\dots < b_n < s$ and errors on rounds $\{1,2,\dots,s-1\} \setminus \{b_1, b_2,
\dots, b_n\}$. Futhermore, suppose
\begin{align*}
(x_{a_1}, y_{a_1}) &= (x'_{b_1}, y'_{b_1}) \; , \\
(x_{a_2}, y_{a_2}) &= (x'_{b_2}, y'_{b_2}) \; , \\
\vdots \\
(x_{a_n}, y_{a_n}) &= (x'_{b_2}, y'_{b_n}) \; .
\end{align*}
Then, as $\calA$ is ignorant,
$$
\Pr[y_t = y ~|~ (x_1, y_1), (x_2, y_2) \dots, (x_{t-1}, y_{t-1}), x_t = x] =
\Pr[y_t' = y ~|~ (x_1', y_1'), (x_2', y_2') \dots, (x_{t-1}', y_{t-1}'), x_t' = x].
$$
Note that the sequences $(x_1, y_1)$, $(x_2, y_2)$, $\dots$, $(x_{t-1},
y_{t-1})$ and $(x_1', y_1')$, $(x_2', y_2')$, $\dots$, $(x_{s-1}', y_{s-1}')$
might have different lengths and and $\calA$ might error in different sets of
rounds. As a special case, if an ignorant algorithm makes a mistake in round $t$
then $p_{t+1}=p_t$.

Our main result is the following lower bound on the expected number of mistakes
for ignorant algorithms.

\begin{theorem}[Mistake lower bound for ignorant algorithms]
\label{theorem:ignorant-lower-bound}
Let $\gamma \in (0,1)$ and let $d$ be a positive integer. Suppose $\calA$ is an
ignorant algorithm for \textsc{Online Multiclass Linear Classification with
Bandit Feedback}. There exists $T$ and an adversary that sequentially chooses
labeled examples $(x_1, y_1), (x_2, y_2), \dots, (x_T, y_T) \in \R^d\times
\{1,2\}$ such that the examples are strongly linearly separable with magin
$\gamma$ and $\norm{x_1}, \norm{x_2}, \dots, \norm{x_T} \le 1$, and the expected
number of mistakes made by $\calA$ is at least
$$
\frac{1}{10} \left(\frac{1}{160\gamma}\right)^{\frac{d-2}{4}} \; .
$$
\end{theorem}

Before proving the theorem, we need the following lemma.

\begin{lemma}
\label{lemma:embed_d_gamma}
Let $\gamma \in (0,\frac{1}{160})$, let $d$ be a positive integer and let $N = (\frac{1}{2\sqrt{40\gamma}})^{d-2}$.
There exist vectors $u_1, u_2, \dots, u_N$, $v_1, v_2, \dots, v_N$ in $\R^d$ such that for all $i, j \in \{1,2,\dots,N\}$,
\begin{align*}
\norm{u_i} & \le 1 \; , \\
\norm{v_j} & \le 1 \; , \\
\ip{u_i}{v_j} & \ge \gamma, \quad \text{if $i=j$,} \\
\ip{u_i}{v_j} & \le -\gamma, \quad \text{if $i \neq j$.}
\end{align*}
\end{lemma}

\begin{proof}
By Lemma 6 of~\citet{Long-1995}, there exists vectors $z_1, z_2, \dots, z_N \in
\R^{d-1}$ such that $\norm{z_1} = \norm{z_2} = \dots = \norm{z_N} = 1$ and the
angle between the vectors is $\measuredangle(z_i, z_j) \ge \sqrt{40 \gamma}$ for
$i \neq j$, $i,j \in \{1,2,\dots,N\}$. Since $\cos\theta \le 1-\theta^2/5$ for
any $\theta \in [-\pi,\pi]$, this implies that
\begin{align*}
\ip{z_i}{z_j} &= 1, \quad \text{if $i = j$,} \\
\ip{z_i}{z_j} &\le 1 - 8\gamma, \quad \text{if $i \neq j$.}
\end{align*}

Define $v_i = (\frac{1}{2} z_i, \frac{1}{2})$, and $u_i = (\frac{1}{2} z_i,
-\frac{1}{2}(1-4\gamma))$ for all $i \in \{1,2,\dots,N\}$. It can be easily
checked that for all $i$, $\norm{v_i} \le 1$ and $\norm{u_i} \le 1$.
Additionally,
$$
\ip{u_i}{v_j} = \frac{1}{4} \ip{z_i}{z_j} - \frac {1-4\gamma} 4 \; .
$$
Thus,
\begin{align*}
\ip{u_i}{v_j} &\ge \gamma, \quad \text{if $i=j$,} \\
\ip{u_i}{v_j} &\le -\gamma, \quad \text{if $i \neq j$.}
\end{align*}
\end{proof}

\begin{proof}[Proof of \autoref{theorem:ignorant-lower-bound}]
We consider the strategy for the adversary described in
Algorithm~\ref{algorithm:adversary-strategy}.

\begin{algorithm}
\caption{\textsc{Adversary's strategy}}
\label{algorithm:adversary-strategy}
\textbf{Define} $T=N$ and $v_1, v_2, \dots, v_N$ as in Lemma~\ref{lemma:embed_d_gamma}.\\
\textbf{Define} $q_0=\frac{1}{\sqrt{T}}$. \\
\textbf{Initialize} $\textsc{phase}= 1$. \\
\For{$t=1,2,\dots,T$}{
    \If{$\textsc{phase}=1$}{
       \If{$p_t(1|v_t) < 1-q_0$}{
          $(x_t, y_t)\leftarrow (v_t, 1)$
        }
       \Else{
          $(x_t, y_t)\leftarrow (v_t, 2)$ \\
          $\textsc{phase}\leftarrow 2$
       }
    }
    \Else{
         $(x_t, y_t)\leftarrow (x_{t-1}, y_{t-1})$
    }
}
\end{algorithm}

Let $\tau$ be the time step $t$ in which the adversary sets $\textsc{phase}\leftarrow 2$.
If the adversary never sets $\textsc{phase}\leftarrow 2$, we define $\tau = T + 1$.
Then,
\begin{align*}
 \Exp \left[\sum_{t=1}^T \indicator{\widehat y_t\neq y_t}\right]
 \ge \Exp\left[\sum_{t=1}^{\tau - 1} \indicator{\widehat y_t\neq y_t}\right]
+ \Exp\left[\sum_{t=\tau}^T \indicator{\widehat y_t\neq y_t}\right] \; .
\end{align*}
We upper bound each of last two terms separately.

In rounds $1,2,\dots,\tau-1$, the algorithm predicts the incorrect class $2$
with probability at least $q_0$. Thus,
\begin{equation}
\Exp\left[\sum_{t=1}^{\tau - 1} \indicator{\widehat y_t\neq y_t}\right] = q_0 \Exp[(\tau - 1)] \; .
\label{eqn:err-phase1}
\end{equation}
In rounds $\tau, \tau+1, \dots, T$, all the examples are the same and are equal
to $(v_\tau, 2)$. Let $s$ be the first time step $t$ such that $t \ge \tau$
and the algorithm makes a correct prediction. If the algorithm makes mistakes
in all rounds $\tau, \tau+1, \dots, T$, we define $s = T+1$.
By definition the algorithm makes mistakes in rounds $\tau, \tau+1, \dots, s-1$.
Therefore,
\begin{equation}
\Exp\left[\sum_{t=\tau}^T \indicator{\widehat y_t\neq y_t}\right] \ge \Exp[s-\tau].
\label{eqn:err-phase2}
\end{equation}
%Since the algorithm is ignorant, conditioned on $\tau$, $s-\tau+1$ has the same distribution as
%$\min(X, T-\tau+2)$, where $X$ is drawn from a geometric distribution with parameter
%$p_\tau(2|v_\tau)$.
Since the algorithm is ignorant, conditioned on $\tau$ and $q \triangleq p_\tau(2|v_\tau)$, $s-\tau$ follows a truncated geometric distribution with parameter $q$
(i.e., $s-\tau$ is $0$ with probability $q$, $1$ with probability $(1-q)q$, $2$ with probability $(1-q)^2q, \ldots$). Its conditional expectation can be calculated as follows:
\begin{align*}
\Exp[s-\tau~|~\tau, q]
&= \sum_{i=1}^{T+1-\tau} i \times \Pr[s-\tau = i |~\tau, q ] \\
&= \sum_{j=1}^{T+1-\tau}\Pr[s-\tau\geq j | ~\tau, q] \\
& = \sum_{j=1}^{T+1-\tau}(1-q)^j \geq  \sum_{j=1}^{T+1-\tau}(1-q_0)^j \\
& = \frac{1-q_0}{q_0}\left( 1-(1-q_0)^{T-\tau+1} \right).
\end{align*}

Therefore, by the tower property of conditional expectation,
\[ \Exp[s-\tau~|~\tau] = \Exp\left[ \Exp\left[s-\tau~\middle|~\tau, q\right]~ \middle| ~\tau\right] \geq \frac{1-q_0}{q_0}\left( 1-(1-q_0)^{T-\tau+1}\right). \]

%= \sum_{i=1}^{T+1-\tau}\sum_{j=1}^i \Pr[s-\tau=i] = \sum_{j=1}^{T+1-\tau}\sum_{i=j}^{T+1-\tau} \Pr[s-\tau=i] \\
%By Lemma~\ref{lem:truncated-geom} below,
%\[
%  \Exp\left[s-\tau+1 \mid \tau, p_\tau(2|v_\tau) \right] = F(p_\tau(2|v_\tau), T-\tau+2) \geq F(q_0, T-\tau+2).
%\]
%Therefore,
%$$
%\Exp[s-\tau ~|~ \tau]
%\ge
%F(q_0, T-\tau+2) - 1
%=
%\frac{1-q_0}{q_0}\left(1-(1-q_0)^{T-\tau+1}\right) \; .
%$$

%the random variable $s-\tau+1$
%has a truncated geometric distribution with parameter $p_\tau(2|v_\tau) \le q_0$.
%Specifically, conditioned on $\tau$,

Combining this fact with Equations~\eqref{eqn:err-phase1} and~\eqref{eqn:err-phase2}, we have that
\begin{align*}
 \Exp \left[\sum_{t=1}^T \indicator{\widehat y_t\neq y_t}\right]
& \ge q_0 \Exp[\tau - 1] + \Exp \left[\frac{1-q_0}{q_0}\left(1-(1-q_0)^{T-\tau+1}\right)\right] \\
& =  \Exp \left[  q_0 (\tau - 1) + \frac{1-q_0}{q_0}\left(1-(1-q_0)^{T-\tau+1}\right)  \right] \; .
\end{align*}

%\Exp \left[~\middle|~ \tau \right]

We lower bound the last expression by considering two cases for $\tau$.
If $\tau \ge \frac{1}{2}T + 1$, then the last expression is lower bounded by
$\frac{1}{2}q_0 T = \frac{1}{2} \sqrt{T}$. If
$\tau < \frac{1}{2}T+1$, it is lower bounded by
\begin{align*}
& \frac{1-q_0}{q_0}\left(1-(1-q_0)^{\frac{1}{2}T}\right) \\
& = \frac{1-q_0}{q_0}\left(1-(1-q_0)^{\frac{1}{2q_0^2}}\right) \\
& \ge \frac{1-\frac{1}{\sqrt{2}}}{q_0}\left(1-\frac{1}{\sqrt{e}}\right) \\
& \ge \frac{1}{10} \sqrt{T} \; .
\end{align*}

Observe that in phase 1, the labels are equal to $1$ and in phase 2 the labels
are equal to $2$. Note that $(x_\tau, y_\tau)=(x_{\tau+1}, y_{\tau+1})= \dots =
(x_T, y_T) = (v_\tau, 2)$. Consider the vectors $u_1, u_2, \dots, u_N$ as
defined in Lemma~\ref{lemma:embed_d_gamma}. We claim that $w_1=-u_\tau/2$ and
$w_2=u_\tau/2$ satisfy the conditions of strong linear separability.

Clearly $\norm{w_1}^2 + \norm{w_2}^2 \le (\norm{w_1} + \norm{w_2})^2 \le
(\frac{1}{2} + \frac{1}{2})^2 \le 1$. By Lemma~\ref{lemma:embed_d_gamma}, we
have $\ip{w_2/2}{x_t} = \ip{u_\tau/2}{v_\tau} \ge \gamma/2, \forall t \ge \tau$ and
$\ip{w_2/2}{x_t} = \ip{u_\tau/2}{v_t} \le - \gamma/2$ for all $t < \tau$. Similarly,
$\ip{w_1/2}{x_t} \le -\gamma/2$ for all $t \ge \tau$ and $\ip{w_1/2}{x_t} \ge
\gamma/2$ for all $t < \tau$. Thus, the examples are strongly linearly
separable with margin $\gamma$.
\end{proof}

%\begin{lemma}
%For $p \in (0,1)$ and $N \geq 1$, denote by $F(p, N) = \frac{1}{p}(1 - (1-p)^N)$.
%\begin{enumerate}
%  \item $F(p, N)$ is monotonically decreasing in $p$.
%  \item If $X \sim \Geom(p)$, then $\Exp[\min(X, N)] = F(p, N)$.
%\end{enumerate}
%\label{lem:truncated-geom}
%\end{lemma}
%\begin{proof}
%We show the two items respectively.
%\begin{enumerate}
%\item By calculus,
%\[ \frac{\partial F(p,n)}{\partial p} = \frac{(1-p)^{N-1} (1 + (N-1)p) - 1}{p^2}. \]
%Now, note that for any $p \in (0,1)$,
%\[ (1 + (N-1)p) \leq \frac{1}{(1 - p)^{N-1}} \]
%We therefore get $\frac{\partial F(p,n)}{\partial p} < 0$ for all $p$, showing the first item.

%\item We have the following decomposition on $\Exp[X]$:
%\[ \Exp[X] = \Exp[X \one[X \leq N]] + \Prob(X \geq N+1) \Exp[X|X \geq N+1] \]
%By the memoryless property of geometric distribution, $\Exp[X|X \geq N+1] = N + \Exp[X] = N + \frac 1 p$.
%Therefore,
%\[ \Exp[X] = \Exp[X \one[X \leq N]] + \Prob(X \geq N+1) \cdot N + \Prob(X \geq N+1) \cdot \frac 1 p. \]
%At the same time, observe that $\Exp[X \one[X \leq N]] + \Prob(X \geq N+1) \cdot N = \Exp[\min(X, N)]$.
%Therefore,
%\[ \Exp[\min(X, N)] = \frac 1 p \cdot (1 - \Prob(X \geq N+1)) = \frac{1}{p}(1 - (1-p)^N) = F(p, N). \qedhere \]
%\end{enumerate}
%\end{proof}

\end{document}